\documentclass{article}
\RequirePackage{graphicx}
\RequirePackage{xcolor}
\RequirePackage{amscd}
\RequirePackage{framed}
\RequirePackage{bbm}
\RequirePackage{amsmath,amsthm,amssymb}
\RequirePackage{fancyhdr,a4wide}
\RequirePackage{comment}
\RequirePackage[dvipdfmx,colorlinks=true,citecolor=blue,urlcolor=blue]{hyperref}

\newcommand{\mc}{\mathcal}
\newcommand{\mbb}{\mathbb}
\newcommand{\mr}{\mathrm}
\newcommand{\argmin}{\mathop{\rm argmin}\limits}

\newcommand{\veca}{\mathbf{a}}

\newcommand{\vecu}{\mathbf{u}}
\newcommand{\vecv}{\mathbf{v}}
\newcommand{\vecw}{\mathbf{w}}
\newcommand{\vecx}{\mathbf{x}}
\newcommand{\vecX}{\mathbf{X}}

\newcommand{\vecz}{\mathbf{z}}
\newcommand{\vecZ}{\mathbf{Z}}

\newcommand{\vecbeta}{\boldsymbol \beta}
\newcommand{\vectheta}{\boldsymbol \theta}

\newcommand{\vecvarrho}{\boldsymbol \varrho}

\RequirePackage{bm}

\RequirePackage{algorithm}
\RequirePackage{algorithmic}
\renewcommand{\algorithmicrequire}{\textbf{Input:}}
\renewcommand{\algorithmicensure}{\textbf{Output:}}

\numberwithin{equation}{section}
\newtheorem{theorem}{Theorem}[section]
\newtheorem{proposition}{Proposition}[section]
\newtheorem{remark}{Remark}[section]
\newtheorem{lemma}{Lemma}[section]

\newtheorem{definition}{Definition}[section]
\newtheorem{assumption}{Assumption}[section]

\begin{document}

\title{Sparse Linear Regression when Noises and Covariates are Heavy-Tailed and Contaminated by Outliers}
\author{Takeyuki Sasai
\thanks{Department of Statistical Science, The Graduate University for Advanced Studies, SOKENDAI, Tokyo, Japan. Email: sasai@ism.ac.jp}
\and Hironori Fujisawa
\thanks{The Institute of Statistical Mathematics, Tokyo, Japan. 
Department of Statistical Science, The Graduate University for Advanced Studies, SOKENDAI, Tokyo, Japan. 
Center for Advanced Integrated Intelligence Research, RIKEN, Tokyo, Japan. Email:fujisawa@ism.ac.jp}
}
\maketitle
\begin{abstract}
	We investigate a problem estimating coefficients of linear regression under sparsity assumption when covariates and noises are sampled from heavy tailed distributions.
	Additionally, we consider the situation where  not only  covariates and noises are sampled from heavy tailed distributions but also  contaminated by outliers. Our estimators can be computed efficiently, and exhibit sharp error bounds.
\end{abstract}

\section{Introduction}
Sparse estimation has been studied extensively over the past 20 years to handle modern high-dimensional data with \cite{Tib1996Regression} as a starting point.
Because the advancement of computer technology has made it possible to collect very high dimensional data efficiently, sparse estimation will continue to be an important and effective method for high dimensional data analysis in the future.
In this study, we focus on the estimation of coefficients in sparse linear regression. We define sparse linear regression model as follows:
\begin{align}
	\label{model:normal}
	y_i = \vecx_i^\top\vecbeta^*+\xi_i,\quad  i=1,\cdots,n,
\end{align}
where $\left\{\vecx_i\right\}_{i=1}^n$ is a sequence of  independent and identically distributed (i.i.d.) random vectors, $\vecbeta^* \in \mbb{R}^d$ is the true coefficient vector, and $\left\{\xi_i\right\}_{i=1}^n$ is a sequence of i.i.d. random variables, and $\vecbeta^*$ is the sparse regression coefficient with $s(< d)$ non-zero elements.

In this paper, we focus on constructing tractable estimators from the observation $\{y_i\,\vecx_i\}_{i=1}^n$, and deriving non-asymptotic error bounds.
This problem has been considered in many literature.
Many studies dealt with the situation where  $\{\vecx_i\}_{i=1}^n$ and $\{\xi_i\}_{i=1}^n$ are sampled from Gaussian distribution and univariate Gaussian distribution, respectively. Some studies weakens the assumption and  deal with the situation where $\{\vecx_i\}_{i=1}^n$ and $\{\xi_i\}_{i=1}^n$ are sampled from a multivariate sub-Gaussian distribution (Chapter 3 of \cite{Ver2018High}) and univariate sub-Gaussian distribution (Chapter 2 of \cite{Ver2018High}), respectively (e.g.,\cite{Wai2009Sharp, RasWaiYu2010Restricted,BelLecTsy2018Slope}). 

Several studies investigated robust estimation methods with respect to the tail heaviness of data.
For example,  \cite{AlqCotLec2019Estimation, SunZhoFan2020Adaptive} considered the case where $\{\xi_i\}_{i=1}^n$ is sampled from a heavy tailed distribution.
However, very few studies tackled  the case where $\{\vecx_i\}_{i=1}^n$ is sampled from a heavier-tailed distribution than Gaussian and subGaussian \cite{SivBenRav2015Beyond,GenKip2022Generic,FanWanZhu2021Shrinkage,LiuLiCar2019High}.
\cite{SivBenRav2015Beyond,GenKip2022Generic} considered the case where $\{\vecx_i\}_{i=1}^n$ is sampled from a multivariate sub-exponential distribution (Chapter 3 of \cite{Ver2018High}), which is a heavier-tailed distribution than multivariate Gaussian.
\cite{FanWanZhu2021Shrinkage,LiuLiCar2019High} considered a more relaxed assumption where $\{\vecx_i\}_{i=1}^n$ is sampled from a finite kurtosis distribution:
\begin{definition}[Finite kurtosis distribution]
	\label{def:fk}
	A random vector $\vecz\in \mbb{R}^d$ is said to be sampled from a finite kurtosis distribution if for every $\vecv \in \mbb{R}^d$, 
	\begin{align}
		\label{ine:fc}
		\mbb{E}\langle \vecv,\vecz -\mbb{E}\vecz \rangle^4 \leq K^4 \left(\mbb{E}\langle \vecv, \vecz - \mbb{E}\vecz \rangle^2\right)^2\left( = K^4 \|\Sigma^\frac{1}{2}\vecv\|_2^4\right),
	\end{align}
	where $K$ is a constant and  $\mbb{E}(\vecz-\mbb{E}\vecz) (\vecz-\mbb{E}\vecz)^\top = \Sigma$.
\end{definition}
We note that the finite kurtosis distribution contains multivariate sub-Gaussian and sub-exponential distributions.
In the present paper,  we assume the finite kurtosis condition on $\{\vecx_i\}_{i=1}^n$ and we construct an estimator which has properties similar to the ones of  \cite{FanWanZhu2021Shrinkage,LiuLiCar2019High}
under different assumption on $\{\xi_i\}_{i=1}^n$ and different conditions on $\vecbeta^*,\,n,\,s,\,d$ and the covariance of $\{\vecx_i\}_{i=1}^n$.

Another aspect of robustness is robustness to outliers. Estimation  against outliers is very actively studied in recent years \cite{DiaKan2023Algorithmic}. 
To investigate the robustness against heavy tailed distributions and outliers for estimating coefficients in sparse linear regression, we consider the following model:
\begin{align}
	\label{model:adv}
	y_i = \vecX_i^\top\vecbeta^*+\xi_i+\sqrt{n}\theta_i,\quad  i=1,\cdots,n,
\end{align}
where $\vecX_i = \vecx_i+\vecvarrho_i$ for $i=1,\cdots,n$ and $\{\vecvarrho_i\}_{i=1}^n$ and $\{\theta_i\}_{i=1}^n$ are the outliers.
We allow the adversary to inject arbitrary values into arbitral $o$ samples of $\{y_i,\vecx_i\}_{i=1}^n$. Let $\mc{O}$ be the index set of the injected samples and $\mc{I} = (1,\cdots,n)\setminus \mc{O}$. Therefore, $\vecvarrho_i= (0,\cdots,0)^\top$ and $\theta_i=0$ hold for $i\in\mc{I}$. We note  that  $\{\vecvarrho_i\}_{i\in \mc{O}}$ and $\{\theta_i\}_{i\in \mc{O}}$ can be arbitral values and they are allowed to correlate freely among them and correlate with $\{\vecx_i\}_{i=1}^n$ and $\{\xi_i\}_{i=1}^n$.
Some studies considered the problem of constructing estimators of $\vecbeta^*$ from \eqref{model:adv}.
\cite{CheCarMan2013Robust,BalDuLiSin2017Computationally,LiuSheLiCar2020High,SasFuj2022Outlier} dealt with the case where $\{\vecx_i\}_{i=1}^n$ is sampled from a Gaussian or multivariate subGaussian distribution.
However, few studies considered  the case  where the covariates are sampled from a heavy tailed distribution and  contaminated by outliers.  An exception to this is \cite{MerGai2023robust}, and the methods in \cite{MerGai2023robust} assume that $\{\vecx_i\}_{i=1}^n$ sampled from a finite kurtosis distribution.
In the present paper, we consider the case where $\{\vecx_i\}_{i=1}^n$ is sampled from a finite kurtosis distribution, and our estimator attains a shaper error bound than the one of \cite{MerGai2023robust}.

In Section \ref{sec:no}, we describe a method to estimate $\vecbeta^*$ from \eqref{model:normal}, and present our result, and introduce preceding work related to our result.
In Section \ref{sec:om}, we describe a method to estimate $\vecbeta^*$ from \eqref{model:adv}, and present our result, and introduce preceding works related to our result.
Proofs of the statements in the main text are given in the appendix.

\section{Method, result and related work}
\label{sec:no}
In Section \ref{sec:no}, we consider the case where there are no outliers.
That is, consider the problem of estimating $\vecbeta^*$ in \eqref{model:normal}.
Before presenting our method and result precisely, we roughly introduce our result.
For linear regression problem \eqref{model:normal} with heavy-tailed $\{\vecx_i,\xi_i\}_{i=1}^n$, our estimator $\hat{\vecbeta}$ satisfies
\begin{align}
	\mbb{P}\left(\|\Sigma^\frac{1}{2}(\hat{\vecbeta} -\vecbeta^*)\|_2 \leq C_{\{\vecx_i,\xi_i\}_{i=1}^n}\sqrt{s\frac{\log (d/\delta)}{n}}\right) \geq 1-\delta,
\end{align}
where $C_{\{\vecx_i,\xi_i\}_{i=1}^n}$ is a constant depending on the moment properties of $\{\vecx_i,\xi_i\}_{i=1}^n$. Our method is based on simple thresholding for $\{\vecx_i\}_{i=1}^n$ and  $\ell_1$-penalized Huber loss, and our estimator is tractable. The result is similar to the one of normal lasso under Gaussian data. The difference between our results and the case of normal lasso under Gaussian data  is discussed in Section \ref{sec:no:re}.

\subsection{Method}
For the problem of estimating $\vecbeta^*$ from  \eqref{model:normal}, we propose ROBUST-SPARSE-ESTIMATION I (Algorithm \ref{nooutlier}).
\begin{algorithm}
	\caption{ROBUST-SPARSE-ESTIMATION I}
	\begin{algorithmic}[1]
			\label{nooutlier}
		\renewcommand{\algorithmicrequire}{\textbf{Input:}}
		\renewcommand{\algorithmicensure}{\textbf{Output:}}
		\REQUIRE $\left\{y_i,\vecx_i\right\}_{i=1}^n$ and the tuning parameters $\tau_\vecx,\,\lambda_o$ and $\lambda_s$
		\ENSURE  $\hat{\vecbeta}$
		\STATE $\{\tilde{\vecx}_i\}_{i=1}^n \leftarrow \text{THRESHOLDING}(\left\{\vecx_i\right\}_{i=1}^n,\tau_\vecx)$
		\STATE $\hat{\vecbeta} \leftarrow \text{PENALIZED-HUBER-REGRESSION}\left(\{y_i,\tilde{\vecx}_i\}_{i=1}^n,\,\lambda_o,\,\lambda_s\right)$
	\end{algorithmic} 
\end{algorithm}

THRESHOLDING is a procedure to make covariates bounded, which is originated from \cite{FanWanZhu2021Shrinkage,KeMinRenZhaQia2019User}. By bounding the covariates through THRESHOLDING, it becomes possible to derive concentration inequalities with sufficient sharpness to handle sparsity.
PENALIZED-HUBER-REGRESSION is a Huber regression with  $\ell_1$-norm  penalization.
From the preceding works \cite{NguTra2012Robust,DalTho2019Outlier,SunZhoFan2020Adaptive}, the $\ell_1$-penalized Huber regression has robustness against heavy-tailed noise ($\{\xi\}_{i=1}^n$) and outliers in the output, while also being capable of handling the sparsity of  the coefficient.

In Sections \ref{sec:t:no} and \ref{sec:phr}, we describe the details of THRESHOLDING and 
PENALIZED-HUBER-REGRESSION, respectively.
\subsubsection{THRESHOLDING}
\label{sec:t:no}
Define the $j$-th element of $\vecx_i$ as $\vecx_{i_j}$.  
THRESHOLDING (Algorithm \ref{alg:t})  makes the covariates bounded to obtain sharp concentration inequalities.
\begin{algorithm}[H]
	\caption{THRESHOLDING}
	\label{alg:t}
	\begin{algorithmic}
	\REQUIRE{data $\{\vecx_i \}_{i=1}^n$ and tuning parameter $\tau_\vecx$.}
	\ENSURE{thresholded data $\{\tilde{\vecx}_i \}_{i=1}^n$.}\\
	{\bf For} {$i=1:n$}\\
		\ \ \ \ \ {\bf For} {$j=1:d$}\\
		\ \ \ \ \ \ \ \ \ $\tilde{\vecx}_{i_j}$ = $\mr{sgn}(\vecx_{i_j})\times \min\left(|\vecx_{i_j}|,\tau_\vecx\right)$\\
	{\bf return}	$\{\tilde{\vecx}_i \}_{i=1}^n$.    
\end{algorithmic}
\end{algorithm}

\subsubsection{PENALIZED-HUBER-REGRESSION}
\label{sec:phr}
PENALIZED-HUBER-REGRESSION (Algorithm \ref{alg:H})  is a type of regression using the Huber loss with $\ell_1$ penalization.
Define the Huber loss function as 
\begin{align}
H(t) = \begin{cases}
|t| -1/2 & (|t| > 1) \\
t^2/2  & (|t| \leq 1)
\end{cases},
\end{align}
and let
\begin{align}
	h(t) =	\frac{d}{dt} H(t) =   \begin{cases}
	t\quad &(|t|> 1)\\
	\mr{sgn}(t)\quad &(|t| \leq 1)
\end{cases}.
\end{align}
We consider the following optimization problem.
For any vector $\vecv$, define the $\ell_1$ norm of $\vecv$ as $\|\vecv\|_1$.
\begin{algorithm}[H]
	\caption{PENALIZED-HUBER-REGRESSION}
	\label{alg:H}
	\begin{algorithmic}
	\REQUIRE{Input data $\left\{y_i,\vecx_i\right\}_{i=1}^n$ and tuning parameters $\lambda_o, \lambda_s$.}
	\ENSURE{estimator $\hat{\vecbeta}$.}\\
	Let $\hat{\vecbeta}$ be the solution to 
	\begin{align}
		\argmin_{\vecbeta  \in \mbb{R}^d} \sum_{i=1}^n \lambda_o^2 H\left(\frac{y_i-\vecx_i^\top\vecbeta}{\lambda_o\sqrt{n}}\right)+\lambda_s\|\vecbeta\|_1,
		\end{align}
	 {\bf return}	$\hat{\vecbeta}$.    
\end{algorithmic}
\end{algorithm}

\subsection{Result}
Before we state our first main result, we introduce the restricted eigenvalue condition of the covariance matrix, which is often used in the context of sparse estimation \cite{BulGee2011Statistics}.
For a vector $\vecv$ and an index set $J$, define  $\vecv_I$ as the vector such that the $i$-th element of $\vecv_I$ and $\vecv$ is equal for $i \in J$ and $i$-th element of $\vecv_J$ is zero for $i \notin J$.
\begin{definition}[Restricted eigenvalue condition of the covariance matrix]
	\label{def:resc}
	A covariance matrix $\Sigma$ is said to satisfy the restricted eigenvalue condition $\mr{RE}(s,c_{\mr{RE}},\kappa)$ with some constants $c_{\mr{RE}}, \kappa >0$ if $\|\Sigma^\frac{1}{2}\vecv\|_2 \geq \kappa \|\vecv_{J}\|_2$ for any  $\vecv\in \mbb{R}^d$ and any set $J\subset \{1,\cdots,d\}$  such that $|J|\leq s$ and $\|\vecv_{J^c}\|_1\leq c_{\mr{RE}}\|\vecv_J\|_1$.
\end{definition}

For a $\vecv \in \mbb{R}^d$, define $\|\vecv\|_0$   as the number of non-zero elements of $\vecv$. 
Then, we state our assumption.
\begin{assumption}
	\label{a:1}
	Assume that 
	\begin{itemize}
		\item[(i)] $\{\vecx_i\}_{i=1}^n$  is a sequence of i.i.d. $d\, (\geq 3)$-dimensional  random vectors with $\mbb{E}\vecx_i = 0$, finite kurtosis with $K$ and  $\max_{1\leq j\leq d} \mbb{E}\vecx_{i_j}^2 \leq 1$. Additionally, $\mbb{E}\vecx_i \vecx_i^\top = \Sigma$ satisfies $\mr{RE}(s,c_{RE},\kappa)$, and $\min_{\vecv \in \mbb{S}^{d-1}, \|\vecv\|_0\leq s} \vecv^\top \Sigma \vecv \geq \kappa^2_{\mr{l}} >0$, 
		\item[(ii)] $\{\xi_i\}_{i=1}^n$ is a sequence of i.i.d. random variables such that $\mbb{E}|\xi_i|\leq \sigma$,
		\item[(iii)] $\mbb{E}h\left(\frac{\xi_i}{\lambda_o\sqrt{n}}\right) \times \vecx_i=0$.
	\end{itemize}
\end{assumption}
Note that assumption (iii) is weaker than independence between $\{\vecx_i\}_{i=1}^n$ and $\{\xi_i\}_{i=1}^n$.
Define $j=1,\cdots,d$, and $\max_{1\leq j\leq d} |\vecbeta^*_j|=c_{\vecbeta}$.
Our first result is the following.
\begin{theorem}
	\label{t:main:no}
	Suppose that Assumption \ref{a:1} holds. 
	Suppose that the parameters $\tau_\vecx\,,\lambda_o$ and $\lambda_s$ satisfy 
	\begin{align}
		\label{ine:tp1:no}
		\tau_\vecx =\sqrt{\frac{n}{\log(d/\delta)}},\quad
		\lambda_o  \sqrt{n} \geq 18K^4(\sigma+1),\quad
		\lambda_s=  c_s \frac{c_{\mr{RE}}+1}{c_{\mr{RE}}-1}\lambda_o \sqrt{n}\sqrt{\frac{\log (d/\delta)}{n}},
	\end{align}
	where  $c_s\geq  16$, and $r_\Sigma,\,r_1$ and $r_2$ satisfy
	\begin{align}
		\label{ine:tp1:no:r}
		r_\Sigma = 12 c_{r_1} \sqrt{s} \lambda_s,\quad
		r_1 =  c_{r_1} \sqrt{s}r_\Sigma, \quad r_2 =  c_{r_2}r_\Sigma,
	\end{align} 
  where $c_{r_1} = c_r(1+c_{\mr{RE}})/\kappa$, $c_{r_2} = c_r(1+c_{\mr{RE}})/\kappa_\mr{l}$ and  $c_r\geq 6$.
	Assume that  $r_\Sigma\leq 1$ and 
	\begin{align}
		\label{ine:tp2:no}
			\max \left\{9K^2 c_{r_1}\sqrt{s\frac{\log(d/\delta)}{n}},\frac{K^4}{c_{r_1}}c_{\vecbeta}^\frac{1}{4} s^\frac{1}{4}\left(\frac{\log(d/\delta)}{n}\right)^{ \frac{7}{8}},\|\vecbeta^*\|_1  K^4 \left(\frac{\log(d/\delta)}{n}\right)^\frac{3}{2}\right\}\leq 1.
	\end{align}
	Then,	with probability at least $1-2\delta$,
	the output of ROBUST-SPARSE-ESTIMATION I $\hat{\vecbeta}$ satisfies
	\begin{align}
		\label{ine:result:no}
		\|\Sigma^\frac{1}{2}(\hat{\vecbeta} -\vecbeta^*)\|_2 & \leq r_\Sigma,\quad \|\hat{\vecbeta} -\vecbeta^*\|_1  \leq r_1,\quad \|\hat{\vecbeta} -\vecbeta^*\|_2  \leq r_2.
	\end{align}
\end{theorem}

\begin{remark}
When we set $c_s = 1$, $c_r=6$ and $\lambda_o  \sqrt{n} = 18K^4(\sigma+1)$ in Theorem \ref{t:main:no}.  Then, the error bounds become
	\begin{align}
		\label{ine:or}
		\|\Sigma^\frac{1}{2}(\hat{\vecbeta} -\vecbeta^*)\|_2 &\lesssim_{C_{\mr{RE}}} \frac{K^4(1+\sigma)}{\kappa}\sqrt{s\frac{\log (d/\delta)}{n}},	\\
		\|\hat{\vecbeta} -\vecbeta^*\|_2 &\lesssim_{C_{\mr{RE}}} \frac{K^4(1+\sigma)}{\kappa \kappa_l}\sqrt{s\frac{\log (d/\delta)}{n}},
	\end{align}
	where $\lesssim_{C_{\mr{RE}}}$ is an inequality up to numerical and $C_{\mr{RE}}$ factor.
\end{remark}
\begin{remark}
Consider the case of normal lasso estimator when  $\{\vecx_i\}_{i=1}^n$ and $\{\xi_i\}_{i=1}^n$ are sampled from Gaussian distributions with  $\mbb{E}\xi_i^2 = \sigma^2$ when $\vecbeta^*$ is estimated by normal lasso estimator.
The error bound of this case is  $ \lesssim_{C_{\mr{RE}}}  \sigma \left(\frac{1}{\kappa}\sqrt{\frac{s\log d}{n}} + \sqrt{\frac{\log(1/\delta)}{n}}\right)$. The result \eqref{ine:or} is similar to this.
 However, the error bound of our estimator does not converge to $0$ when $\sigma \to 0$, and our estimator has a cross term such that $\sqrt{s} \times \sqrt{\frac{\log (1/\delta)}{n}}$.
Whether  one can construct tractable estimators without these limitations in our situation is a future work.
\end{remark}
\subsection{Related work}
\label{sec:no:re}
In this section, we introduce related work which dealt with estimating coefficients in sparse linear regression by tractable estimator under finite kurtosis condition without outlier contamination.

One of the tractable estimation methods proposed in \cite{FanWanZhu2021Shrinkage}  can be applied to estimate $\vecbeta^*$ from the data $\{y_i,\vecx_i\}_{i=1}^n$ is derived from \eqref{model:normal}. 
Let the obtained estimator be denoted by $\hat{\vecbeta}_1$ and $\lambda_{\min}$ as smallest eigenvalue of $\Sigma$. Then, $\hat{\vecbeta}_1$ demonstrates the following error bound: for any $\gamma>0$, 
\begin{align}
\label{ine:preceed1}
	\mbb{P}\left(\|\hat{\vecbeta}_1-\vecbeta^*\|_2 \lesssim  \frac{1}{\lambda_{\min}}\sqrt{\frac{\gamma s\log d}{n}}\right)\geq 1-d^{1-\gamma},
\end{align}
when $\mbb{E}\vecx_i = 0$,  $\{\vecx_i\}_{i=1}^n$ and $\{\xi_i\}_{i=1}^n$ are independent, $\|\vecbeta^*\|_1$, $\mbb{E}(\vecx_i^\top  \vecbeta^*+\xi_i)^4$ and $s\sqrt{\frac{\log d}{n}}$ are sufficiently small.
The advantage of our results compared to \eqref{ine:preceed1} lies in more relaxed assumption on $\vecbeta^*$, the moment of the noise, the minimum eigenvalue of $\Sigma$ and sample complexity.

One of the estimation methods introduced in \cite{LiuLiCar2019High} can also  be applicable. 
The estimator $\hat{\vecbeta}_2$  demonstrates the following error bound:
\begin{align}
	\label{ine:preceed2}
	\mbb{P}\left(\|\hat{\vecbeta}_2-\vecbeta^*\|_2 \lesssim \rho ^\frac{5}{2}\sigma \sqrt{\frac{s\log d}{n}}\right)\geq 1-\frac{1}{d^2},
\end{align}
under the condition  $\mbb{E}\xi_i=0$, $\mbb{E}\xi_i^2 = \sigma^2$ and $\rho^7 s (\log d) \times \log \left( \frac{\mu_\alpha \|\vecbeta^*\|_2}{\rho^\frac{3}{2}\sigma}\sqrt{\frac{n}{ s \log d}}\right)\lesssim n$, where $\rho = \mu_L/\mu_\alpha$, and $\mu_L$ and $\mu_\alpha$ are `smoothness' and `strong-convexity' parameters, respectively.
Depending on the shape of the covariance matrix, `smoothness' and `strong-convexity' parameters can be significantly large and small, respectively, and \cite{LiuLiCar2019High} does not explicitly evaluate these parameters.
The explicit evaluation of the effect of the covariance and more relaxed moment assumption on $\{\xi_i\}_{i=1}^n$ is the advantage of our results over \eqref{ine:preceed2}. On the other hand, the error bound in  \eqref{ine:preceed2} has exact recovery when $\sigma \to 0$.

None of the methods, including ours, have been able to eliminate the dependence on $\|\vecbeta^*\|_2$. Constructing an estimator that is independent of $\|\vecbeta^*\|_2$ while preserving the exact recovery when $\sigma \to 0$ remains a challenge for future work.

\section{Case of contamination}
\label{sec:om}
Before presenting our method and result precisely, we roughly introduce our result.
For linear regression problem \eqref{model:adv} with heavy-tailed $\{\vecx_i,\xi_i\}_{i=1}^n$, our estimator  satisfies
\begin{align}
	\mbb{P}\left(\|\Sigma^\frac{1}{2}(\hat{\vecbeta} -\vecbeta^*)\|_2 \leq C_{\{\vecx_i,\xi_i\}_{i=1}^n}\left(\sqrt{s\frac{\log (d/\delta)}{n}}+\sqrt{\frac{o}{n}}\right)\right) \geq 1-\delta,
\end{align}
for sufficiently large $n$ such that $s^2 \lesssim_{\{\vecx_i,\xi_i\}_{i=1}^n, \|\vecbeta^*\|_1}n$, where $\lesssim_{\{\vecx_i,\xi_i\}_{i=1}^n, \|\vecbeta^*\|_1}$ is a inequality up to the moment properties of $\{\vecx_i,\xi_i\}_{i=1}^n$ and $\|\vecbeta^*\|_1$. Our method is based on simple thresholding for $\{\vecx_i\}_{i=1}^n$, mitigating the impact of outliers on $\{\vecx_i\}_{i=1}^n$ by robust sparse PCA and $\ell_1$-penalized Huber loss, and our estimator is tractable. 
The impact of outliers on the error bound depends solely on the proportion of outliers.
However, the requirement for the number of samples $n$ being proportional to the square of the sparsity $s^2$ is different from the usual lasso. This point is discussed in Section \ref{sec:remark:s2}.

\subsection{Method}
To estimate $\vecbeta^*$ in \eqref{model:adv}, we propose ROBUST-SPARSE-ESTIMATION II (Algorithm \ref{ourmethod}), which is an extension of ROBUST-SPARSE-ESTIMATION-I.
ROBUST-SPARSE-ESTIMATION II  is inspired by a method proposed in \cite{PenJogLoh2020robust}.
\cite{PenJogLoh2020robust} proposed some methods for estimating $\vecbeta^*$  when $\vecbeta^*$ has no sparsity, and derived sharp error bounds. One of the methods in \cite{PenJogLoh2020robust} consists of two steps: (i) pre-processing covariates, and (ii) executing the Huber regression with pre-processed covariates.
Our method is based on this one. However, we follow different pre-processings (THRESHOLDING and COMPUTE-WEIGHT) and use the $\ell_1$-penalized Huber regression to enable us to tame the sparsity of $\vecbeta^*$. 
\begin{algorithm}
	\caption{ROBUST-SPARSE-ESTIMATION II}
	\begin{algorithmic}[1]
			\label{ourmethod}
		\renewcommand{\algorithmicrequire}{\textbf{Input:}}
		\renewcommand{\algorithmicensure}{\textbf{Output:}}
		\REQUIRE $\left\{y_i,\vecX_i\right\}_{i=1}^n$ and the tuning parameters $\tau_\vecx,\,\lambda_*,\,\,\tau_{suc},\,\varepsilon,\,\lambda_o$ and $\lambda_s$
		\ENSURE  $\hat{\vecbeta}$
		\STATE $\{\tilde{\vecX}_i\}_{i=1}^n \leftarrow \text{THRESHOLDING}(\left\{\vecX_i\right\}_{i=1}^n,\tau_\vecx)$
		\STATE $\left\{ \hat{w}_i\right\}_{i=1}^n \leftarrow \text{COMPUTE-WEIGHT}(\{\tilde{\vecX}_i\}_{i=1}^n ,\lambda_*,\tau_{suc},\varepsilon )$
		\STATE $\hat{\vecbeta} \leftarrow \text{PENALIZED-HUBER-REGRESSION}\left(\{\hat{w}_iy_i,\hat{w}_i\tilde{\vecX}_i\}_{i=1}^n, \,\lambda_o,\,\lambda_s\right)$
	\end{algorithmic} 
\end{algorithm}

COMPUTE-WEIGHT relies on the semi-definite programming developed by \cite{BalDuLiSin2017Computationally}, which provides a method for  sparse PCA to be robust to outliers.
\cite{BalDuLiSin2017Computationally} considered a situation where samples  are drawn from a Gaussian distribution and contaminated by outliers. THRESHOLDING enables us to cast our heavy tailed situation into the framework of \cite{BalDuLiSin2017Computationally}.
In Section \ref{sec:cw}, we describe the details of  COMPUTE-WEIGHT.

\subsubsection{COMPUTE-WEIGHT}
\label{sec:cw}
Define the $j$-th element of $\vecX_i$ as $\vecX_{i_j}$.
For a matrix $M\in \mbb{R}^{d_1\times d_2} = \{m_{ij}\}_{1\leq i\leq d_1,1\leq j\leq d_2}$, define
\begin{align}
	\|M\|_1 = \sum_{i=1}^{d_1}\sum_{j=1}^{d_2}|m_{ij}|,\,\quad\|M\|_\infty =\max_{1\leq i\leq d_1,1\leq j\leq d_2}|m_{ij}|.
\end{align}
For a symmetric matrix $M$, we write $M\succeq 0$ if $M$ is positive semidefinite.
For a vector $\vecv$, define the $\ell_\infty$ norm of $\vecv$ as $\|\vecv\|_\infty$,  and for a matrix $M$, define  the absolute maximum element of $M$ as  $\|M\|_\infty$.
Define the following two convex sets:
\begin{align}
	\mathfrak{M}_r = \left\{M\in \mbb{R}^{d\times d} \,:\, \mr{Tr}(M) \leq  r^2,\,M\succeq 0\right\},\quad \mathfrak{U}_{\lambda} = \left\{U\in \mbb{R}^{d\times d} \,:\, \|U\|_{\infty} \leq \lambda,\, U\succeq 0\right\},
\end{align}
where $\mr{Tr}(M)$ for matrix $M$ is the trace of $M$.
To reduce the effects of outliers of covariates,
we require COMPUTE-WEIGHT (Algorithm \ref{alg:cw0}) to compute the weight vector $\hat{\vecw}= (\hat{w}_1,\cdots, \hat{w}_n)$ such that the following quantity is sufficiently small: 
\begin{align}
	\label{ine:adv-spect}
	\sup_{M\in \mathfrak{M}_{r}}\left(\sum_{i=1}^n\hat{w}_i \langle \tilde{\vecX}_i\tilde{\vecX}_i^\top M\rangle-\lambda_* \|M\|_1\right), 
\end{align}
where $\lambda_*$ is a tuning parameter.
Evaluation of  \eqref{ine:adv-spect} is required in the analysis of WEIGHTED-PENALIZED-HUBER-REGRESSION and the role of \eqref{ine:adv-spect} is revealed in the proof of Proposition \ref{p:main:out}.
For  COMPUTE-WEIGHT, we use a variant of Algorithm 4 of \cite{BalDuLiSin2017Computationally}.
define the probability simplex $\Delta^{n-1}$ as
\begin{align}
	\Delta^{n-1} = \left\{\vecw \in [0,1]^n: \sum_{i=1}^nw_i =1, \quad \|\vecw\|_\infty\leq \frac{1}{n(1-\varepsilon)}\right\}.
\end{align}
COMPUTE-WEIGHT is as follows.
\begin{algorithm}[H]
	\caption{COMPUTE-WEIGHT}
	\label{alg:cw0}
	\begin{algorithmic}
	\REQUIRE{data $\{\tilde{\vecX}_i \}_{i=1}^n$ and tuning parameters $\lambda_*,\, \tau_{suc}$ and $\varepsilon$.}
	\ENSURE{weight estimate $\hat{\vecw} = \{\hat{w}_1,\cdots,\hat{w}_n\}$.}\\
	Let  $\hat{\vecw}$ be the solution to 
	\begin{align}
		\label{cw}
		\min_{\vecw \in \Delta^{n-1}} \max_{M\in \mathfrak{M}_r}\left(\sum_{i=1}^n w_i \langle\tilde{\vecX}_i \tilde{\vecX}_i^\top,M\rangle-\lambda_*\|M\|_1\right)
	\end{align}
	{\bf if} {the optimal value of \eqref{cw} $\leq \tau_{suc}$}\\
	\ \ \ \ \ {\bf return} {$\hat{\vecw}$}\\
	{\bf else}    \\
	\ \ \ \ \ {\bf return} {$fail$}\\
\end{algorithmic}
\end{algorithm}
We note that, from the arguments of \cite{WanBerSam2016Statistical, Nem2004Prox,Nes2005Smooth}, we have 
\begin{align}
	\label{ine:opt-transform}
		\min_{\vecw \in \Delta^{n-1}}\max_{M\in \mathfrak{M}_{r}} \left( \sum_{i=1}^n w_i\langle \tilde{\vecX}_i \tilde{\vecX}_i^\top,M\rangle-\lambda_*\|M\|_1\right)= \min_{\vecw \in \Delta^{n-1}} \min_{U \in \mathfrak{U}_{\lambda_*}}\max_{M\in \mathfrak{M}_{r}}\left\langle \sum_{i=1}^n w_i \tilde{\vecX}_i \tilde{\vecX}_i^\top- U,M\right\rangle.
\end{align}
COMPUTE-WEIGHT and Algorithm 4 of \cite{BalDuLiSin2017Computationally} are very similar.
For any fixed $\vecw$, our objective function and the constraints are the same as the ones in Section 3 of \cite{WanBerSam2016Statistical} except for the values of the tuning parameters, and we can efficiently find the optimal $M \in \mathfrak{M}_{r}$.
Therefore, COMPUTE-WEIGHTS can be solved efficiently for the same reason as Algorithm 4 of \cite{BalDuLiSin2017Computationally}.

To analyze COMPUTE-WRIGHT, we introduce the following proposition.
The poof of the following proposition is provided in the appendix (Section \ref{sec:proof:sec:keypropositions}). 

\begin{proposition}
	\label{p:cwpre}
	Assume (i) of Assumption \ref{a:1} holds.
	For any matrix $M \in \mathfrak{M}_{r}$, with probability at least $1-\delta$, we have
	\begin{align}
		\label{ine:cwpre}
	\sum_{i=1}^n \frac{\left\langle\tilde{\vecx}_i \tilde{\vecx}_i^\top,M\right\rangle}{n}\leq  \left(\sqrt{2}K^2\sqrt{\frac{\log(d/\delta)}{n}}+\tau_{\vecx}^2 \frac{\log(d/\delta)}{n}+2\frac{K^4}{\tau_{\vecx}^2}\right)\|M\|_1+\|\Sigma\|_{\mr{op}}r^2,
	\end{align}
	where $\|\Sigma\|_{\mr{op}}$ is the operator norm of $\Sigma$.
\end{proposition}

We consider the succeeding condition of Algorithm \ref{alg:cw0}.
Define $\lambda_*$ and $\{w_i^\circ \}_{i=1}^n$ as
\begin{align}
	\label{ine:*}
	\lambda_*= c_*\left(\sqrt{2}K^2\sqrt{\frac{\log(d/\delta)}{n}}+\tau_{\vecx}^2 \frac{\log(d/\delta)}{n}+2\frac{K^4}{\tau_{\vecx}^2}\right),\,	w_i^\circ = \begin{cases}
		\frac{1}{n(1-\varepsilon)}&i\in\mc{I} \\
		0 & i\in\mc{O}
		\end{cases},
\end{align}
where $c_*\geq \frac{1}{1-\varepsilon}$.
Then, we have, with probability at least $1-\delta$, 
\begin{align}
	\label{ine:optM}
	\max_{M\in \mathfrak{M}_{r}}\left(\sum_{i =1}^n \hat{w}_i \langle \tilde{\vecX}_i\tilde{\vecX}_i^\top, M\rangle -\lambda_*\|M\|_1\right)&\stackrel{(a)}{\leq} \max_{M\in \mathfrak{M}_{r}}\left(\sum_{i \in \mc{I}}w^\circ_i \left\langle \tilde{\vecX}_i\tilde{\vecX}_i^\top ,M\right\rangle -\lambda_*\|M\|_1\right)\nonumber\\
	&=\max_{M\in \mathfrak{M}_{r}}\left(\sum_{i \in \mc{I}}w_i^\circ  \left\langle \tilde{\vecx}_i\tilde{\vecx}_i^\top, M\right\rangle -\lambda_*\|M\|_1\right)\nonumber\\
	&\stackrel{(b)}\leq\max_{M\in \mathfrak{M}_{r}}\left(\sum_{i =1}^n \frac{1}{n(1-\varepsilon)}  \left\langle \tilde{\vecx}_i\tilde{\vecx}_i^\top, M\right\rangle -\lambda_*\|M\|_1\right)\nonumber\\
	&=\max_{M\in \mathfrak{M}_{r}}\left(\sum_{i =1}^n\frac{1}{n}  \left\langle \tilde{\vecx}_i\tilde{\vecx}_i^\top, M\right\rangle -\lambda_*(1-\varepsilon)\|M\|_1  \right) \times \frac{1}{1-\varepsilon}\nonumber\\
	&\stackrel{(c)}{\leq} \frac{\|\Sigma\|_{\mr{op}}}{1-\varepsilon}r^2,
\end{align}
where (a) follows from the optimality of $\hat{w}_i$, $o/n\leq \varepsilon$ and $\{w_i^\circ\}_{i=1}^n \in \Delta^{n-1}$,  (b) follows from positive semi-definiteness of $M$, and (c) follows from Proposition \ref{p:cwpre} and the definition of $\lambda_*$. Therefore, when $\tau_{suc}\geq \frac{\|\Sigma\|_{\mr{op}}}{1-\varepsilon}r^2$, we see that COMPUTE-WEIGHT succeed and return $\hat{\vecw}$ with probability at least $1-\delta$.
We note that \eqref{ine:optM} is used in the proof of the following proposition, which is plays an important role in the proof of Theorem \ref{t:main}.

\begin{proposition}
	\label{p:main:out}
	Assume that (i) of Assumption \ref{a:1} holds and $\lambda_*$ satisfies  \eqref{ine:*}. Let $r_\Sigma$ be some positive constant, and  $r_1 =  c_{r_1} \sqrt{s}r_\Sigma, \, r=r_2 =  c_{r_2}r_\Sigma$,  where $c_{r_1} = c_r(1+c_{\mr{RE}})/\kappa$, $c_{r_2} = c_r(1+c_{\mr{RE}})/\kappa_\mr{l}$ and  $c_r\geq 6$. Further, let  $\varepsilon = c_{\varepsilon}\frac{o}{n}$ and $\tau_{suc} = \frac{\|\Sigma\|_{\mr{op}}}{1-\varepsilon}r_2^2$, 
	where $c_{\varepsilon} \geq 1$. Lastly, assume $1-\varepsilon>0$, and define $c_{*}' = \max\{\frac{1}{1-\varepsilon},c_*\} $.
	Suppose that  \eqref{ine:optM} holds and COMPUTE-WEIGHT returns $\hat{w}$.
	Then, for any  $\|\vecu\|\in \mbb{R}^n$ such that $\|\vecu\|_\infty \leq c$ for a numerical constant $c$ and for any $\vecv \in r_1 \mbb{B}^d_1 \cap r_2 \mbb{B}^d_2 \cap r_\Sigma \mbb{B}^d_\Sigma$, we have
	\begin{align}
		\left|\sum_{i \in \mc{O}}\hat{w}_iu_i \tilde{\vecX}_i^\top\vecv \right| \leq 2cc_*'\left( c_{r_2}\|\Sigma^\frac{1}{2}\|_{\mr{op}}\sqrt{\frac{o}{n}}  +  c_{r_1}\sqrt{s}\sqrt{K^2\sqrt{\frac{\log(d/\delta)}{n}}+\tau_{\vecx}^2 \frac{\log(d/\delta)}{n}+\frac{K^4}{\tau_{\vecx}^2}} \sqrt{\frac{o}{n}}\right)r_\Sigma,
	\end{align}
\end{proposition}
Proof of the proposition above in the appendix (Section \ref{sec:proof:sec:keypropositions}).

\subsection{Result}
\label{sec:result}
Under Assumption \ref{a:1}, we have the following theorem.
\begin{theorem}
	\label{t:main}
	Suppose that Assumption \ref{a:1}, and assumptions in Proposition \ref{p:main:out}  holds. As Proposition \ref{p:main:out},  let $\varepsilon = c_{\varepsilon}\frac{o}{n}$, where $c_{\varepsilon} \geq 1$ and  $c_{*}' = \max\{\frac{1}{1-\varepsilon},c_*\} $.
	Suppose that the parameters $\lambda_o,\,\tau_\vecx$ and $\lambda_s$ satisfy $\lambda_o  \sqrt{n} \geq 18K^4(\sigma+1)$, 
	\begin{align}
		\label{ine:tp1}
		\tau_\vecx&=\left(\frac{n}{\log(d/\delta)}\right)^\frac{1}{4},\, \lambda_s=  c_s \frac{c_{\mr{RE}}+1}{c_{\mr{RE}}-1}\lambda_o\sqrt{n}\left(\sqrt{\frac{\log(d/\delta)}{n}}+ c_*' \|\Sigma^\frac{1}{2}\|_{\mr{op}} \frac{c_{r_2}}{c_{r_1}}\sqrt{\frac{o}{sn}}\right),
	\end{align}
	where  $c_s\geq 16$, and $r_\Sigma$, $r_1$, $r$ and $r_2$ satisfies
	\begin{align}
		\label{ine:tp1:r}
		r_\Sigma =  12 c_{r_1}\sqrt{s}\lambda_s,\quad r_1 =  c_{r_1} \sqrt{s}r_\Sigma,\quad r=r_2 =  c_{r_2}r_\Sigma,
	\end{align}
	where  $c_{r_1} = c_r(1+c_{\mr{RE}})/\kappa$, $c_{r_2} = c_r(1+c_{\mr{RE}})/\kappa_\mr{l}$ and  $c_r\geq 6$.
	and $\tau_{suc} =  \frac{\|\Sigma\|_{\mr{op}}}{1-\varepsilon}r_2^2$.
	Assume that  $r_\Sigma\leq 1$ and 
	\begin{align}
		\label{ine:tp2}
		\max \left\{ K^4 \left(\frac{\log(d/\delta)}{n}\right)^\frac{1}{4}, \frac{K^4 }{c_{r_1}}c_{\vecbeta}^\frac{1}{4}s^\frac{1}{4}\left(\frac{\log(d/\delta)}{n}\right)^{ \frac{3}{16}},		\frac{3K^4c_{r_1}^2}{c_{r_2}^2\|\Sigma^\frac{1}{2}\|_{\mr{op}}} s \sqrt{\frac{\log(d/\delta)}{n}} \right\}\leq 1.
	\end{align}
	and
	\begin{align}
		\label{ine:tp2-2}
		\max \left\{K^4  \|\vecbeta^*\|_1 \left(\frac{\log(d/\delta)}{n}\right)^\frac{3}{4},72 K^4c_{r_1}^2 s \sqrt{\frac{\log(d/\delta)}{n}}\right\}\leq 1.
	\end{align}
	Then,	with probability at least $1-3\delta$,
	the output of ROBUST-SPARSE-ESTIMATION II $\hat{\vecbeta}$ satisfies
	\begin{align}
		\label{ine:result}
		\|\Sigma^\frac{1}{2}(\hat{\vecbeta} -\vecbeta^*)\|_2 & \leq r_\Sigma,\quad \|\hat{\vecbeta} -\vecbeta^*\|_1  \leq r_1,\quad \|\hat{\vecbeta} -\vecbeta^*\|_2  \leq r_2.
	\end{align}
\end{theorem}

\begin{remark}
When we set $c_s = 16$, $c_r=6$, $\lambda_o  \sqrt{n} = 18K^4(\sigma+1)$ and $c_*' = 2$. Then, the error bounds become
	\begin{align}
		\|\Sigma^\frac{1}{2}(\hat{\vecbeta} -\vecbeta^*)\|_2& \lesssim_{C_{\mr{RE}}}  K^4(1+\sigma)\left(\frac{1}{\kappa} \sqrt{\frac{s\log (d/\delta)}{n}}+ \frac{\|\Sigma^\frac{1}{2}\|_{\mr{op}}}{\kappa_{l}}\sqrt{\frac{o}{n}}\right),\\
		\|\hat{\vecbeta} -\vecbeta^*\|_2& \lesssim_{C_{\mr{RE}}}  K^4(1+\sigma)\left(\frac{1}{\kappa \kappa_l} \sqrt{\frac{s\log (d/\delta)}{n}}+ \frac{\|\Sigma^\frac{1}{2}\|_{\mr{op}}}{\kappa_{l}^2}\sqrt{\frac{o}{n}}\right),
	\end{align}
	for sufficiently large $n$ such that $s^2 \lesssim_{\{\vecx_i,\xi_i\}_{i=1}^n, \|\vecbeta^*\|_1}n$.
 When $d=1$, this result is nearly optimal, however not optimal because the term containing $o$ in optimal convergence rate is $(o/n)^\frac{3}{4}$ \cite{BakPra2021Robust}.
 Whether  one can construct optimal and tractable estimator in our situation is a future work.
\end{remark}
\begin{remark}
\label{sec:remark:s2}
	In Theorem \ref{t:main}, we require $n$ is sufficiently larger than $s^2\log(d/s)$, and this is a stronger condition than the ones of Theorem \ref{t:main:no}. This phenomenon is due to sparse PCA in COMPUTE-WEIGHT.
	\cite{WanBerSam2016Statistical} revealed that there is no randomized polynomial time algorithm to estimate the top eigenvector in a scheme where $n$ is proportional to $s$ (in \cite{WanBerSam2016Statistical}, $s$ is the number of non-zero elements of the top eigenvector of covariance matrices) under the assumptions of intractability of a variant of Planted Clique Problem.  In addition to this research, other studies such as  \cite{DinKunWeiBan2023Subexponential,BreBreHopLiSch2020Statistical} suggested that a sparse PCA may inevitably have a dependence on $n$ of $s^2$ even in the absence of outliers when using polynomial-time algorithms.
	Therefore, with our framework, it would be difficult to avoid the dependence on $s^2$ of $n$, even with a different analysis from ours.
	Recently, various studies on sparse estimation problems  have emerged, suggesting that the unnatural dependence on 
 $s^2$ might be unavoidable for polynomial-time algorithms, using frameworks such as SQ lower bounds and low-degree tests (i.e., \cite{LofWeiBan2022Computationally,SchWei2022Computational,ArpVen2023Statistical,FanLiuWanYan2018Curse,DiaKanSte2017Statistical}). These works considered sparse clustering, mixed sparse regression, robust sparse mean estimation, tensor pca. Deriving tradeoffs between statistical   and computational tradeoffs   for problem similar to ours remains a future challenge.
\end{remark}
\subsection{Related work}
\label{sec:re}
There are not many papers dealing with robust sparse estimation problem from data following heavy-tailed distributions with outliers. For instance, \cite{DiaKanKarPenPit2022Robust, DiaKanLeePen2022Outlier} addressed  estimation problem of  sparse mean vector, while \cite{MerGai2023robust} dealt with sparse linear regression. Here, we will provide a detailed introduction to \cite{MerGai2023robust}, which shares a common problem setting with our paper.
\cite{MerGai2023robust} proposed some robust gradient methods for some objective functions and they are applicable to estimating  $\vecbeta^*$ from the data $\{y_i,\vecx_i\}_{i=1}^n$  from \eqref{model:adv} with slightly strong assumption on $\{\vecvarrho_i,\theta_i\}$. That is to say, the assumption is $\{\vecx_i,\xi_i\}_{i \in \mc{I}}$ remains i.i.d..
Applying the method in \cite{MerGai2023robust} for squared loss as the objective function, after $T$ step starting from $\vecbeta^0$ with $\|\vecbeta^0-\vecbeta^*\|_1\leq R$, we have an estimator $\hat{\vecbeta}_3$  such that, 
\begin{align}
	\mbb{P}\left(\|\hat{\vecbeta}_3-\vecbeta^*\|_2 \lesssim \frac{R}{2^{T/2}\sqrt{s}}+ \frac{(1+\sigma^2)\sqrt{s}}{\min_{\vecv \in \mbb{S}^{d-1}, \|\vecv\|_0\leq 2s} \vecv^\top \Sigma \vecv}\left(\sqrt{\frac{ \log (d/\delta')}{n'}}+\sqrt{\frac{o}{n'}}\right)\right)\geq 1-\delta,
\end{align}
where $n' = n/T$ and $\delta'= \delta/T$, under the condition that finite kurtosis for $\{\vecx_i\}_{i=1}^n$, $\mbb{E}\|y_i \vecx_i\|_2^2$ and $\mbb{E}y_i^2$ are infinite.
Note that, their gradient method is applicable for different objective functions and this condition can be made weaker.
Our error bound is  sharper from the perspective of convergence rate because  our result avoid the coss term such that $\sqrt{so/n}$.
However, the method in \cite{MerGai2023robust} has some merits than ours. For example, the sample complexity is smaller than ours and their method does not require the assumption such that $s^2 \lesssim n$.  
The computational cost required in \cite{MerGai2023robust} is significantly lighter than ours, and the extensive numerical experiments in  \cite{MerGai2023robust} demonstrate the practical effectiveness of their method.

Constructing an estimator with sharp error bounds that can be computed with light computational burden is considered a very interesting challenge.

\bibliographystyle{plain}
\bibliography{SELRCNCHTCO}

\appendix
\section{Key propositions for Theorems \ref{t:main:no} and Theorems \ref{t:main} }
\label{sec:keypropositions:no}
In this section, we provide four propositions needed to prove  Theorems \ref{t:main:no} and  \ref{t:main}.
First, we introduce Proposition \ref{p:main}, which is used in the proof of both Theorems \ref{t:main:no} and  \ref{t:main}.
Proposition \ref{p:main} pertains to the $\ell_1$-penalized Huber loss and is stated without considering the randomness of the input and output, i.e., it does not take Assumption \ref{a:1} into account. The applicability of Proposition \ref{p:main} under Assumption \ref{a:1} is ensured by the subsequent propositions.

\begin{proposition}
	\label{p:main}
		Define the input data for PENALIZED-HUBER-REGRESSION  as $\{\vecZ_i,Y_i\}_{i=1}^n$.
		Suppose that, for any $\vecv \in r_1 \mbb{B}^d_1 \cap r_2 \mbb{B}^d_2 \cap r_\Sigma \mbb{B}^d_\Sigma$, 
		\begin{align}
		\label{ine:det:main:0-1}
		&\left| \lambda_o\sqrt{n}\sum_{i=1}^n \frac{1}{n}h\left(\frac{Y_i-\langle \vecZ_i, \vecbeta^*\rangle}{\lambda_o \sqrt{n}}\right) \vecZ_i^\top \vecv \right|\leq r_{a,\Sigma} r_\Sigma,
		\end{align}
		and for any $\vecv \in r_1 \mbb{B}^d_1 \cap r_2 \mbb{B}^d_2$ such that $\|\Sigma^\frac{1}{2}\vecv \|_2 = r_\Sigma$,
		\begin{align}
			\label{ine:det:main:0-2}
			&b \|\Sigma^\frac{1}{2}\vecv\|_2^2- r_{b,\Sigma }r_\Sigma \leq \lambda_o\sqrt{n} \sum_{i=1}^n \frac{1}{n}\left(-h\left(\frac{Y_i-\langle \vecZ_i,\vecbeta^*+\vecv \rangle}{\lambda_o \sqrt{n}}\right) +h\left(\frac{Y_i-\langle\vecZ_i, \vecbeta^*\rangle}{\lambda_o \sqrt{n}}\right)\right) \vecZ_i^\top \vecv,
		\end{align}
		where $r_{a,\Sigma}, r_{b,\Sigma}\geq 0$ and $b>0$.
		Suppose that $\mbb{E}\vecx_i\vecx_i^\top=\Sigma$ satisfies $\mr{RE}(s,c_{\mr RE},\kappa)$ with $\kappa_{l}>0$, and 
		\begin{align}
		\label{ine:det:main:0-3}
			&\lambda_s-\frac{r_{a,\Sigma}}{c_{r_1}\sqrt{s}} >0,\quad\frac{\lambda_s +  \frac{ r_{a,\Sigma}}{c_{r_1}\sqrt{s}}}{\lambda_s - \frac{ r_{a,\Sigma}}{c_{r_1}\sqrt{s}} }\leq c_{\mr{RE}},\\
		\label{ine:det:main:0-4}
			&r_\Sigma \geq \frac{2}{b} \left( r_{a,\Sigma}+r_{b,\Sigma} + c_{r_1}\sqrt{s}\lambda_s\right),\quad r_1 =  c_{r_1}\sqrt{s}r_\Sigma\quad r_2 = c_{r_2} r_\Sigma
		\end{align}
		hold, where $c_{r_1} = c_r(1+c_{\mr{RE}})/\kappa$, $c_{r_2} = c_r(1+c_{\mr{RE}})/\kappa_{l}$ and $c_r\geq 6$.
		Then, we have the following:
		\begin{align}
			\|\vecbeta^*-\hat{\vecbeta}\|_1\leq r_1,\quad \|\vecbeta^*-\hat{\vecbeta}\|_2\leq r_2,\quad \|\Sigma^\frac{1}{2}(\vecbeta^*-\hat{\vecbeta})\|_2\leq r_\Sigma.
		\end{align}
	\end{proposition}
	In \cite{SasFuj2023Robust}, a proposition almost identical to Proposition \ref{p:main} was introduced, and for completeness, we provide a proof.
We note that Proposition \ref{p:main} is a modification of the claim found in \cite{LugMen2019Risk,AlqCotLec2019Estimation} to deal with the case where the covariance of the covariate has a general form (not identity).
The proof of Proposition \ref{p:main} is provided in Sections \ref{sec:mainproposition-pre} and \ref{sec:mainproof}.

Second, we introduce Propositions \ref{p:main1:no} and \ref{p:main:sc:no}.
By these propositions, we see that   \eqref{ine:det:main:0-1} and \eqref{ine:det:main:0-2}
are satisfied with high probability for appropriate values of $b>0, r_{a,\Sigma},r_{b,\Sigma}, r_1, r_2, r_\Sigma \geq 0$ under the assumptions in Theorem \ref{t:main:no}.
\begin{proposition}
	\label{p:main1:no}
	Suppose that the Assumption \ref{a:1} holds and $r_1 =  c_{r_1} \sqrt{s}r_\Sigma, \, r_2 =  c_{r_2}r_\Sigma$, where $c_{r_1},\,c_{r_2}$ are some positive constants.
	Then, for any $\vecv \in r_1 \mbb{B}^d_1 \cap r_2 \mbb{B}^d_2 \cap r_\Sigma \mbb{B}^d_\Sigma$, with probability at least $1-\delta$, we have  
	\begin{align}
		\label{ine:main1:no}
		&\left|\sum_{i=1}^n  \frac{1}{n}h\left(\frac{\langle \vecx_i-\tilde{\vecx}_i, \vecbeta^*\rangle+\xi_i}{\lambda_o \sqrt{n}}\right)\langle \tilde{\vecx}_i,\vecv\rangle \right|\nonumber\\
		&\leq 4 \left(c_{r_1}\sqrt{s\frac{\log (d/\delta)}{n}}+c_{r_1}\tau_\vecx\sqrt{s}\frac{\log (d/\delta)}{n}+K^4c_{r_1}\frac{\sqrt{s}}{\tau_{\vecx}^3}+  K^4 c_{\vecbeta}^\frac{1}{4}s^\frac{3}{4}\left(\frac{1}{\tau_{\vecx}}\right)^{ 3-\frac{1}{4}}\right)r_\Sigma.
	\end{align}
\end{proposition}

\begin{proposition}
	\label{p:main:sc:no}
	Suppose that the Assumption \ref{a:1} holds and $r_1 =  c_{r_1} \sqrt{s}r_\Sigma, \, r_2 =  c_{r_2}r_\Sigma$, where $c_{r_1},\,c_{r_2}$  are some positive constants. Assume that 	$\lambda_o  \sqrt{n} \geq 18K^4(\sigma+1)$.
	Then, for any $\vecv \in r_1 \mbb{B}^d_1 \cap r_2 \mbb{B}^d_2 $ such that $\|\Sigma^\frac{1}{2}\vecv\|_2 = r_\Sigma$, where $r_\Sigma\leq 1$, with probability at least $1-\delta$,  we have
	\begin{align}
		\label{ine:sc:no}
		&\lambda_o \sqrt{n}\sum_{i=1}^n \frac{1}{n} \left(-h\left(\frac{\langle \vecx_i-\tilde{\vecx}_i, \vecbeta^*\rangle+\xi_i}{\lambda_o \sqrt{n}}-\frac{\tilde{\vecx}_i^\top \vecv}{\lambda_o\sqrt{n}}\right)+h\left(\frac{\langle \vecx_i-\tilde{\vecx}_i, \vecbeta^*\rangle+\xi_i}{\lambda_o \sqrt{n}}\right)  \right)\tilde{\vecx}_i^\top \vecv\nonumber \\
		&\quad \geq  \|\Sigma^\frac{1}{2}\vecv\|_2^2\left( 1-K^2\sqrt{\frac{2}{\lambda_o\sqrt{n}}} \left(\sqrt{\sigma+\|\vecbeta^*\|_1\frac{K^4}{\tau_{\vecx}^3}}+\sqrt{1+c_{r_1}\sqrt{s}\frac{K^4}{\tau_{\vecx}^3}}\right)-12K^4 c_{r_1}^2\frac{s} {\tau_{\vecx}^2}\right)\nonumber\\
		&\quad \quad -3 \lambda_o\sqrt{n} c_{r_1} \left(1+\frac{K^2}{\tau_\vecx}+\tau_\vecx\sqrt{\frac{\log (d/\delta)}{n}}\right) \sqrt{\frac{s\log(d/\delta)}{n}}\|\Sigma^\frac{1}{2}\vecv\|_2.
	\end{align}
\end{proposition}

To prove Theorem \ref{t:main}, we need the following propositions   (Propositions \ref{p:main:out2} and \ref{l:w2}), in addition to the propositions above.
The poof of the propositions are provided in Section \ref{sec:proof:sec:keypropositions}.
Define $I_m$ as the index set such that $|I_m|$ = $m$.
\begin{proposition}
	\label{p:main:out2}
	Suppose that \eqref{ine:cwpre} holds.
	For any  $\vecu \in \mbb{R}^n$ such that $\|\vecu\|_\infty \leq c$ for a numerical constant $c$ and for any $\vecv  \in r_1 \mbb{B}^d_1 \cap r_2 \mbb{B}^d_2 \cap r_\Sigma \mbb{B}^d_\Sigma$, we have
	\begin{align}
		\left|\sum_{i \in I_m}\frac{1}{n}u_i \tilde{\vecx}_i^\top\vecv \right|\leq  2c\left(c_{r_2}\|\Sigma^\frac{1}{2}\|_{\mr{op}}\sqrt{\frac{m}{n}} + c_{r_1}\sqrt{s}\sqrt{K^2\sqrt{\frac{\log(d/\delta)}{n}}+\tau_{\vecx}^2 \frac{\log(d/\delta)}{n}+\frac{K^4}{\tau_{\vecx}^2}}\sqrt{\frac{m}{n}}\right)r_\Sigma.
	\end{align}
\end{proposition}

Let $I_{<}$ and $I_{\geq}$ be  the sets of the indices such that $w_i < 1/(2n)$  and $w_i \geq 1/(2n)$, respectively. 
\begin{proposition}
	\label{l:w2}
	Suppose $0<\varepsilon <1 $. Then, for any $\vecw \in \Delta^{n-1}$, we have $|I_{<}| \leq 2n\varepsilon$.
\end{proposition}

\section{Auxiliary lemmas}
For an event $E$, define the indicator function of $E$ as $\mr{I}_{E}$.
For a vector $\vecv \in \mbb{R}^d$, we denote the $i$-th element of the vector as $\vecv_i$.
In this section, we introduce some useful lemmas.
\begin{lemma}
	\label{a:thex}
	Define $z \in \mbb{R}$ as a random variable  such that $\mbb{E}z^4 \leq \sigma_{z,4}^4$. Then we have
	\begin{align}
		\mbb{E}\mr{I}_{|z|\geq \tau_{z}}\leq  \frac{\sigma_{z,4}^4}{\tau_{z}^4}.
	\end{align}
\end{lemma}
\begin{proof}
	\begin{align}
		\mbb{E}\mr{I}_{|z|\geq \tau_{z}}&\stackrel{(a)}{\leq} \mbb{P}\left(|z|\geq \tau_{z}\right)\stackrel{(b)}{\leq}  \mbb{E}\frac{|z|^4}{\tau_{z}^4}\leq \frac{\sigma_{z,4}^4}{\tau_{z}^4},
	\end{align}
	where (a) follows from the relationship between expectation and probability, and (b) follows from Markov's inequality.
\end{proof}

\begin{lemma}
	\label{a:l:1}
	Define $\vecz \in \mbb{R}^d$ as a random vector  such that $\max_{1\leq j\leq d}\mbb{E} \vecz_j^4\leq \sigma_{\vecz,4}^4$, and define $\tilde{\vecz}$ as  a random vector such that for any $j \in \{1,\cdots,d\}$, $\tilde{\vecz}_j = \mr{sgn}(\vecz_j)\min(\tau_\vecz,|\vecz_j|)$ with a positive constant $\tau_\vecz$. Then for any $\vecv \in \mbb{R}^d$, we have
	\begin{align}
		-4\frac{\sigma_{\vecz,4}^4} {\tau_{\vecz}^2}\|\vecv\|_1^2+ \mbb{E}(\vecz^\top \vecv)^2	\leq \mbb{E}(\tilde{\vecz}^\top \vecv)^2 &\leq  4\frac{\sigma_{\vecz,4}^4} {\tau_{\vecz}^2}\|\vecv\|_1^2+ \mbb{E}(\vecz^\top \vecv)^2.
	\end{align}
\end{lemma}
\begin{proof}
	From simple algebra, we have
	\begin{align}
		\mbb{E}(\tilde{\vecz}^\top \vecv)^2 &= \mbb{E}\langle \tilde{\vecz} \tilde{\vecz}^\top , \vecv \vecv^\top \rangle\nonumber\\
		&= \mbb{E}\langle \tilde{\vecz} \tilde{\vecz}_i^\top-\vecz \vecz^\top , \vecv \vecv^\top \rangle+\mbb{E}\langle \vecz \vecz^\top , \vecv \vecv^\top \rangle\nonumber\\
		&= \mbb{E}\langle \tilde{\vecz} \tilde{\vecz}^\top-\vecz \vecz^\top , \vecv \vecv^\top \rangle+\mbb{E}(\vecz^\top \vecv)^2,
	\end{align}
	and we have
	\begin{align}
		\label{eq0:a:l:1}
		- |\mbb{E}\langle \tilde{\vecz} \tilde{\vecz}^\top-\vecz \vecz^\top , \vecv \vecv^\top \rangle|+\mbb{E}(\vecz^\top \vecv)^2\leq  \mbb{E}(\tilde{\vecz}^\top \vecv)^2\leq |\mbb{E}\langle \tilde{\vecz} \tilde{\vecz}^\top-\vecz \vecz^\top , \vecv \vecv^\top \rangle|+\mbb{E}(\vecz^\top \vecv)^2,
	\end{align}
	For any $1\leq j,k,\leq d$, we have 
	\begin{align}
		\label{eq1:a:l:1}
		|\mbb{E}\left(\tilde{\vecz}_{j}\tilde{\vecz}_{k} -\vecz_{j}\vecz_{k}\right)| &=  |\mbb{E}\left(\tilde{\vecz}_{j}(\tilde{\vecz}_{k}-\vecz_{k}) +\vecz_{k}(\tilde{\vecz}_{j}-\vecz_{j}) \right)| \nonumber\\
		&\leq  |\mbb{E}\tilde{\vecz}_{j}(\tilde{\vecz}_{k}-\vecz_{k}) +|\mbb{E}\vecz_{k}(\tilde{\vecz}_{j}-\vecz_{j}) | \nonumber\\
		&\leq  \mbb{E}|\tilde{\vecz}_{j}||\tilde{\vecz}_{k}-\vecz_{k}| +\mbb{E}|\vecz_{k}||\tilde{\vecz}_{j}-\vecz_{j} | \nonumber\\
		&=  \mbb{E}|\tilde{\vecz}_{j}||\tilde{\vecz}_{k}-\vecz_{k}|\mr{I}_{|\vecz_{k}|>\tau_{\vecz}} +\mbb{E}|\vecz_{k}||\tilde{\vecz}_{j}-\vecz_{j} |\mr{I}_{|\vecz_j|>\tau_{\vecz}} \nonumber\\
		&\leq   \mbb{E}|\vecz_{j}||\tilde{\vecz}_{k}-\vecz_{k}|\mr{I}_{|\vecz_{k}|>\tau_{\vecz}} +\mbb{E}|\vecz_{k}||\tilde{\vecz}_{j}-\vecz_{j} |\mr{I}_{|\vecz_j|>\tau_{\vecz}} \nonumber\\
		&\leq  2 \mbb{E}|\vecz_{j}||\vecz_{k}|\mr{I}_{|\vecz_{k}|>\tau_{\vecz}} +2\mbb{E}|\vecz_{k}||\vecz_{j} |\mr{I}_{|\vecz_j|>\tau_{\vecz}} \nonumber\\
		&\leq  2\sqrt{ \mbb{E}\vecz_{j}^2 \vecz_{k}^2} \sqrt{\mbb{E}\mr{I}_{|\vecz_{k}|>\tau_{\vecz}}} + 2\sqrt{\mbb{E}\vecz_{k}^2\vecz_{j}^2}\sqrt{\mr{I}_{|\vecz_{j}|>\tau_{\vecz}}} \nonumber\\
		&=  4\sqrt{ \mbb{E}\vecz_{j}^2 \vecz_{k}^2} \sqrt{\mbb{E}\mr{I}_{|\vecz_{k}|>\tau_{\vecz}}}\leq 4\frac{\sigma_{\vecz,4}^4} {\tau_{\vecz}^2},
	\end{align}
	where the last inequality follows from Lemma \ref{a:thex} and $\max_{1\leq j\leq d}\mbb{E} \vecz_j^4\leq \sigma_{\vecz,4}^4$.

	From \eqref{eq0:a:l:1} and  \eqref{eq1:a:l:1}, we have
	\begin{align}
		\label{eq3:a:l:1}
		-4\frac{\sigma_{\vecz,4}^4} {\tau_{\vecz}^2}\|\vecv\|_1^2+ \mbb{E}(\vecz_i^\top \vecv)^2	\leq \mbb{E}(\tilde{\vecz}^\top \vecv)^2 &\leq  4\frac{\sigma_{\vecz,4}^4} {\tau_{\vecz}^2}\|\vecv\|_1^2+ \mbb{E}(\vecz^\top \vecv)^2.
	\end{align}
\end{proof}

\begin{lemma}
	\label{a:l:2}
	Define $\vecz \in \mbb{R}^d$ as a random vector  such that $\max_{1\leq j\leq d}\mbb{E} \vecz_j^4\leq \sigma_{\vecz,4}^4$, and define $\tilde{\vecz}$ as a random vector such that for any $j \in \{1,\cdots,d\}$, $\tilde{\vecz}_j =\mr{sgn}(\vecz_j)\min(\tau_\vecz,|\vecz_j|)$ with a positive constant $\tau_\vecz$. Then for any $\vecv \in \mbb{R}^d$, we have
	\begin{align}
		\mbb{E}|\langle\tilde{\vecz}-\vecz,\vecv\rangle| \leq 2\|\vecv\|_1 \frac{\sigma_{\vecz,4}^4}{\tau_{\vecz}^3}
	\end{align}
\end{lemma}
\begin{proof}
	From simple algebra, we have
	\begin{align}
		\label{eq0:a:l:2}
		\mbb{E}|\langle\tilde{\vecz}-\vecz,\vecv\rangle|  &=  \mbb{E} \left|\sum_{1\leq j\leq d} (\tilde{\vecz}_{j}-\vecz_{j})\vecv_j\right|\nonumber\\
		&\leq  \sum_{1\leq j\leq d}  \mbb{E}\left|(\tilde{\vecz}_{j}-\vecz_{j})\vecv_j\right|\nonumber\\
		&\leq 2\sum_{1\leq j\leq d}  \mbb{E}\left|\vecz_{j}\mr{I}_{|\vecz_{j}|\geq \tau_{\vecz}}\vecv_j\right|\nonumber\\
		&\stackrel{(a)}{=} 2\sum_{1\leq j \leq d} \left( \mbb{E}\left|\vecz_{j}\vecv_j\right|^4\right)^\frac{1}{4} \left(\mbb{E}\mr{I}_{\vecz_{j}\geq \tau_{\vecz}}\right)^\frac{3}{4}\nonumber\\
		&\stackrel{(b)}{=2}\sum_{1 \leq j\leq d} |\sigma_{\vecz,4}\vecv_j| \left(\frac{\sigma_{\vecz,4}^4}{\tau_{\vecz}^4} \right)^\frac{3}{4}=2\|\vecv\|_1 \frac{\sigma_{\vecz,4}^4}{\tau_{\vecz}^3},
	\end{align}
	where  (a) follows from H{\"o}lder's inequality, and (b) follows from Lemma \ref{a:thex}.
\end{proof}

\begin{lemma}
	\label{l:1e}
	Define $\{\vecz_i\}_{i=1}^n $ as a sequence of i.i.d. $d$-dimensional random vector  such that $\max_{1\leq j\leq d}\mbb{E} \vecz_{i_j}^2\leq 1$. For any $1\leq i\leq n$ and $1\leq j\leq d$ define $\tilde{\vecz}_{i_j} = \mr{sgn}(\vecz_{i_j})\min(\tau_\vecz,|\vecz_{i_j}|)$ with a positive constant $\tau_\vecz$. 
	Define $\{\alpha_i\}_{i=1}^n$ as a sequence of i.i.d. Rademacher random variables which are independent of $\{\tilde{\vecz}_i\}_{i=1}^n$. Then  we have
	\begin{align}
		\label{ine:gc-normale}
		&\mbb{E} \sup_{\vecv\in r_1 \mbb{B}^d_1 } \left|\frac{1}{n}\sum_{i=1}^n \alpha_i \tilde{\vecz}_i^\top \vecv\right| \leq \sqrt{2\frac{\log d}{n}}r_1+\tau_\vecz\frac{\log d}{n}r_1.
	\end{align}
\end{lemma}
\begin{proof}

	We note that $
		\mbb{E} \alpha_i^2 \tilde{z}_{i_j}^2 \leq 1$ and $\mbb{E} \alpha_i^p |\tilde{z}_{i_j}|^p\leq 
		\tau_\vecz^{p-2}\mbb{E} \tilde{z}_{i_j}^2  \leq \tau_\vecz^{p-2}$. Then, we have 
		$\frac{1}{n}\sum_{i=1}^n \mbb{E} \alpha_i^p |\tilde{z}_{i_j}|^p\leq \tau_\vecz^{p-2}$.
  From Lemma 14.12 of \cite{BulGee2011Statistics} and $d\geq 3$, we have
	\begin{align}
		 \mbb{E}\left\|\frac{1}{n}\sum_{i=1}^n \alpha_i\tilde{\vecz}_i\right\|_\infty \leq \sqrt{2\frac{\log d}{n}}+\tau_\vecz\frac{\log d}{n}.
	\end{align}
	Lastly, from H{\"o}lder's inequality, the proof is complete.
\end{proof}

\section{Preparation of proof of Proposition \ref{p:main}}
\label{sec:mainproposition-pre}
For  $\eta \in (0,1)$, let
\begin{align}
	\vectheta = \hat{\vecbeta}-\vecbeta^*,\quad \vectheta_\eta = (\hat{\vecbeta}-\vecbeta^*)\eta.
\end{align}
We introduce the following four lemmas, that are used in the proof of Proposition \ref{p:main}.
\begin{lemma}
\label{l:mainpre1-1}
 Suppose that \eqref{ine:det:main:0-1}, \eqref{ine:det:main:0-3}, $\|\vectheta_\eta\|_1\leq  r_1$,  $\|\vectheta_\eta\|_2= r_2$ and  $\|\Sigma^\frac{1}{2}\vectheta_\eta\|_2\leq r_\Sigma$, where $r_1 = c_{r_1}\sqrt{s}r_\Sigma $ and  $r_2 =c_{r_2} r_\Sigma $  hold.
Then, for any fixed $\eta\in(0,1)$, we have
\begin{align}
 \|\vectheta_\eta\|_2  &\leq 3 \frac{1+c_{\mr{RE}}}{\kappa_{l}}\|\Sigma^\frac{1}{2}\vectheta_\eta\|_2.
\end{align}
\end{lemma}
\begin{lemma}
	\label{l:mainpre1-2}
	 Suppose that \eqref{ine:det:main:0-1}, \eqref{ine:det:main:0-3}, $\|\vectheta_\eta\|_1 = r_1$, $\|\vectheta_\eta\|_2\leq  r_2$ and $\|\Sigma^\frac{1}{2}\vectheta_\eta\|_2\leq  r_\Sigma$, where $r_1 = c_{r_1}\sqrt{s}r_\Sigma $ and  $r_2 =c_{r_2} r_\Sigma $ hold.
	Then, for any fixed $\eta\in(0,1)$, we have
	\begin{align}
\|\vectheta_\eta\|_1 &\leq 3 \frac{1+c_{\mr{RE}}}{	\kappa } \sqrt{s}\|\Sigma^\frac{1}{2}\vectheta_\eta\|_2.
	\end{align}
	\end{lemma}

	\begin{lemma}
		\label{l:mainpre2-1}
		Suppose that \eqref{ine:det:main:0-1}, \eqref{ine:det:main:0-2},  $\|\vectheta_\eta\|_1\leq  r_1$,  $\|\vectheta_\eta\|_2\leq r_2$ and  $\|\Sigma^\frac{1}{2}\vectheta_\eta\|_2= r_\Sigma$, where  $r_1 = c_{r_1}\sqrt{s}r_\Sigma $ and  $r_2 =c_{r_2} r_\Sigma $ hold.
		Then, for any  $\eta\in(0,1)$, we have
		\begin{align}
			\|\Sigma^\frac{1}{2}\vectheta_\eta\|_2 &\leq \frac{1}{b} \left(r_{a,\Sigma}+r_{b,\Sigma} + c_{r_1}\sqrt{s}\lambda_s\right).
		\end{align}
	\end{lemma}

\subsection{Proof of Lemma \ref{l:mainpre1-1}}
For a vector $\vecv=(v_1,\cdots,v_d) $, define $\{v^{\#}_1,\cdots,v^{\#}_d\}$ as a non-increasing rearrangement of $\{|v_1|,\cdots,|v_d|\}$, and $\vecv^{\#}\in \mbb{R}^d$ as a vector such that $\vecv^{\#}|_i = v^{\#}_i$. 
For the sets $S_1=\{1,...,s\}$ and $S_2=\{s+1,...,d\}$, let $v^{\#1}=v^\#_{S_1}$ and $v^{\#2}=v^\#_{S_2}$.

In Section \ref{sec:case1:re}, we have
\begin{align}
	\|\vectheta_\eta\|_2 \leq \frac{\sqrt{1+c_{\mr{RE}}}}{\kappa_{l}} \|\Sigma^\frac{1}{2}\vectheta_\eta\|_2
\end{align}
assuming $\|\vectheta_\eta \|_2 \leq \|\vectheta_\eta\|_1/\sqrt{s}$, and in  Section \ref{sec:case2:re}, we have
\begin{align}
	\|\vectheta_\eta\|_2 \leq  2\|\vectheta_\eta^{\#1}\|_2\leq \frac{2}{\kappa_{l}}\|\Sigma^\frac{1}{2}\vectheta_\eta\|_2
\end{align}
assuming $\|\vectheta_\eta \|_2 \geq \|\vectheta_\eta\|_1/\sqrt{s}$. From the above two inequalities, we have 
\begin{align}
	\|\vectheta_\eta\|_2  &\leq \frac{3+c_{\mr{RE}}}{\kappa_{l}}\|\Sigma^\frac{1}{2}\vectheta_\eta\|_2.
 \end{align}

	\subsubsection{Case I}
	\label{sec:case1:re}
	In Section \ref{sec:case1:re}, suppose that $\|\vectheta_\eta \|_2 \leq \|\vectheta_\eta\|_1/\sqrt{s}$.
	Let 	
	\begin{align}
	 Q'(\eta) = \lambda_o\sqrt{n} \frac{1}{n}\sum_{i=1}^n \left(-h\left(\frac{Y_i-\langle \vecZ_i, \vecbeta^*+\vectheta_\eta\rangle}{\lambda_o \sqrt{n}}\right)+h\left(\frac{Y_i-\langle \vecZ_i, \vecbeta^*\rangle}{\lambda_o \sqrt{n}}\right) \right)\vecZ_i^\top \vectheta.
	 \end{align}
	From the proof of Lemma F.2. of \cite{FanLiuSunZha2018Lamm}, we have $\eta Q'(\eta) \leq \eta Q'(1)$ and this means
	\begin{align}
	\label{ine:det:lemma:1}
	\lambda_o\sqrt{n}&\sum_{i=1}^n  \frac{1}{n}\left(-h\left(\frac{Y_i-\langle \vecZ_i, \vecbeta^*+\vectheta_\eta\rangle}{\lambda_o \sqrt{n}}\right)+h\left(\frac{Y_i-\langle \vecZ_i, \vecbeta^*\rangle}{\lambda_o \sqrt{n}}\right) \right) \vecZ_i^\top \vectheta_\eta \nonumber \\
	&\leq \lambda_o\sqrt{n}\sum_{i=1}^n \frac{1}{n}\eta \left(-h\left(\frac{Y_i-\langle \vecZ_i, \hat{\vecbeta}\rangle}{\lambda_o \sqrt{n}}\right)+h\left(\frac{Y_i-\langle \vecZ_i, \vecbeta^*\rangle}{\lambda_o \sqrt{n}}\right) \right)\vecZ_i^\top\vectheta.
	\end{align}
	Let $\partial \vecv$ be the sub-differential of $\|\vecv\|_1$.
	Adding $\eta \lambda_s(\|\hat{\vecbeta}\|_1-\|\vecbeta^*\|_1) $ to both sides of \eqref{ine:det:lemma:1}, we have
	\begin{align}
	\label{ine:det:lemma:2}
	&\lambda_o\sqrt{n}\sum_{i=1}^n\frac{1}{n}\left(-h\left(\frac{Y_i-\langle \vecZ_i, \vecbeta^*+\vectheta_\eta\rangle}{\lambda_o \sqrt{n}}\right)+h\left(\frac{Y_i-\langle \vecZ_i, \vecbeta^*\rangle}{\lambda_o \sqrt{n}}\right) \right)\vecZ_i^\top\vectheta_\eta +\eta \lambda_*(\|\hat{\vecbeta}\|_1-\|\vecbeta^*\|_1)\nonumber \\
	 & \stackrel{(a)}{\leq}\lambda_o\sqrt{n} \sum_{i=1}^n \frac{1}{n}\eta \left(-h\left(\frac{Y_i-\langle \vecZ_i, \hat{\vecbeta}\rangle}{\lambda_o \sqrt{n}}\right)+h\left(\frac{Y_i-\langle \vecZ_i, \vecbeta^*\rangle}{\lambda_o \sqrt{n}}\right) \right)\vecZ_i^\top\hat{\vectheta}+\eta \lambda_s\langle \partial \hat{\vecbeta},\vectheta\rangle\nonumber\\
	 &\stackrel{(b)}{=} \lambda_o\sqrt{n}\sum_{i=1}^n\frac{1}{n} h\left(\frac{Y_i-\langle \vecZ_i, \vecbeta^*\rangle}{\lambda_o \sqrt{n}}\right) \vecZ_i^\top\vectheta_\eta,
	\end{align}
	where (a) follows from $\|\hat{\vecbeta}\|_1-\|\vecbeta^*\|_1 \leq \langle \partial \hat{\vecbeta},\vectheta\rangle$, which is the definition of the sub-differential, and (b) follows from the optimality of $\hat{\vecbeta}$.

	From the convexity of Huber loss, the first term of the left hand side of \eqref{ine:det:lemma:2} is positive and we have
	\begin{align}
		\label{ine:det:lemma:3}
	0\leq \sum_{i=1}^n \lambda_o\sqrt{n} \frac{1}{n} h\left(\frac{Y_i-\langle \vecZ_i, \vecbeta^*\rangle}{\lambda_o \sqrt{n}}\right)\vecZ_i^\top\vectheta_\eta+	\eta \lambda_s(\|\vecbeta^*\|_1-\|\hat{\vecbeta}\|_1).
	\end{align}
From \eqref{ine:det:main:0-1}, the first term of the right-hand side of \eqref{ine:det:lemma:3} is evaluated as
	\begin{align}
		\label{ine:det:lemma:4}
		\sum_{i=1}^n \lambda_o\sqrt{n} \frac{1}{n} h\left(\frac{Y_i-\langle \vecZ_i, \vecbeta^*\rangle}{\lambda_o \sqrt{n}}\right)\vecZ_i^\top\vectheta_\eta\leq  r_{a,\Sigma}r_\Sigma\stackrel{(a)}{\leq}  \frac{r_{a,\Sigma}}{c_{r_2}}r_2\stackrel{(b)}{=} \frac{r_{a,\Sigma}}{c_{r_2}}\|\vectheta_\eta\|_2\stackrel{(c)}{\leq}   \frac{r_{a,\Sigma}}{c_{r_2}\sqrt{s}}\|\vectheta_\eta\|_1,
	\end{align}
	where (a) follows from  $r_2 =c_{r_2} r_\Sigma $, (b) follows from the assumption $\|\vectheta_\eta\|_2= r_2$   and (c) follows from the assumption $\|\vectheta_\eta\|_2 \leq \|\vectheta_\eta\|_1/\sqrt{s}$.
	From~\eqref{ine:det:lemma:3},~\eqref{ine:det:lemma:4} and the assumption $\|\vectheta_\eta \|_2 \leq \|\vectheta_\eta\|_1/\sqrt{s}$, we have
	\begin{align}
	0 \leq \frac{r_{a,\Sigma}}{c_{r_2}\sqrt{s}} \|\vectheta_\eta\|_1+\eta \lambda_s (\|\vecbeta^*\|_1-\|\hat{\vecbeta}\|_1).
	\end{align}
	Define $\mc{J}_{\veca}$ as the index set of  non-zero entries of $\veca$, and $\vectheta_{\eta, \mc{J}_{\veca}}$ as a vector such that $\vectheta_{\eta, \mc{J}_{\veca}}|_i=\vectheta_\eta|_i$ for $i \in \mc{J}_{\veca}$ and $\vectheta_{\eta, \mc{J}_{\veca}}|_i = 0$ for $i \notin \mc{J}_{\veca}$.
	Furthermore, we see
	\begin{align}
	0
	&\leq \frac{r_{a,\Sigma}}{c_{r_2}\sqrt{s}} \|\vectheta_\eta\|_1+\eta \lambda_s(\|\vecbeta^*\|_1-\|\hat{\vecbeta}\|_1) \nonumber \\
	& \leq \frac{r_{a,\Sigma}}{c_{r_2}\sqrt{s}}  (\|\vectheta_{\eta,\mc{J}_{\vecbeta^*}}\|_1+\|\vectheta_{\eta,\mc{J}^c_{\vecbeta^*}}\|_1) + \eta\lambda_s(\|\vecbeta^*_{\mc{J}_{\vecbeta^*}}- \hat{\vecbeta}_{\mc{J}_{\vecbeta^*}}\|_1-\|\hat{\vecbeta}_{\mc{J}^c_{\vecbeta^*}}\|_1)\nonumber\\
	& =\left(\lambda_s + \frac{r_{a,\Sigma}}{c_{r_2}\sqrt{s}}  \right)\|\vectheta_{\eta,\mc{J}_{\vecbeta^*}}\|_1
	+\left(-\lambda_s+ \frac{r_{a,\Sigma}}{c_{r_2}\sqrt{s}}  \right)\|\vectheta_{\eta,\mc{J}^c_{\vecbeta^*}}\|_1.
	\end{align}
	Then, we have
		\begin{align}
			\| \vectheta_{\eta,\mc{J}_{\vecbeta^*}^c}\|_1 \leq \frac{\lambda_s + \frac{r_{a,\Sigma}}{c_{r_2}\sqrt{s}} }{\lambda_s -  \frac{r_{a,\Sigma}}{c_{r_2}\sqrt{s}}  }\| \vectheta_{\eta,\mc{J}_{\vecbeta^*}}\|_1\stackrel{(a)}{\leq}\frac{\lambda_s +  \frac{ r_{a,\Sigma}}{c_{r_1}\sqrt{s}}}{\lambda_s -  \frac{ r_{a,\Sigma}}{c_{r_1}\sqrt{s}} }\| \vectheta_{\eta,\mc{J}_{\vecbeta^*}}\|_1 \stackrel{(b)}{\leq} c_{\mr{RE}} \| \vectheta_{\eta,\mc{J}_{\vecbeta^*}}\|_1,
		\end{align}
		where (a) follows from the fact that $c_{r_2}\geq c_{r_1}$ and (b) follows from \eqref{ine:det:main:0-3}, and
		from the definition of $\| \vectheta_\eta^{\#2}\|_1 $ and $\| \vectheta_\eta^{\#1}\|_1$, we have
		\begin{align}
			\label{ine:det:lemma:5}
			\| \vectheta_\eta^{\#2}\|_1 \leq c_{\mr{RE}} \| \vectheta_\eta^{\#1}\|_1.
		\end{align}
	Then, from the standard shelling argument, we have
	\begin{align}
		\| \vectheta_\eta^{\#2}\|_2^2=\sum_{i=s+1}^d(\vectheta_{\eta}^{\#}|_i)^2 \leq \sum_{i=s+1}^d\left| \vectheta_{\eta}^{\#}|_i\right|\left(\frac{1}{s}\sum_{j=1}^s\left|\vectheta_{\eta}^{\#}|_j\right|\right) \leq \frac{1}{s}\| \vectheta_\eta^{\#1}\|_1\| \vectheta_\eta^{\#2}\|_1\leq \frac{c_{\mr{RE}}\| \vectheta_\eta^{\#1}\|_1^2}{s}\leq c_{\mr{RE}}\| \vectheta_\eta^{\#1}\|_2^2.
	\end{align}
	and from the definition of $\kappa_{l}$, we have 
	\begin{align}
		\kappa_{l}^2 \|\vectheta_\eta\|_2^2 \leq \kappa_{l}^2 \left(\|\vectheta_\eta^{\#1}\|_2^2 +\|\vectheta_\eta^{\#2}\|_2^2\right)\leq \kappa_{l}^2 (1+c_{\mr{RE}})\|\vectheta_\eta^{\#1}\|_2^2 \leq(1+c_{\mr{RE}}) \|\Sigma^\frac{1}{2}\vectheta_\eta\|_2^2.
	\end{align}
		\subsubsection{Case II}
		\label{sec:case2:re}
		In Section \ref{sec:case2:re}, suppose that $\|\vectheta_\eta \|_2 \geq \|\vectheta_\eta\|_1/\sqrt{s}$.
		\begin{align}
			\| \vectheta_\eta^{\#2}\|_2^2=\sum_{i=s+1}^d(\vectheta_{\eta}^{\#}|_i)^2 \leq \sum_{i=s+1}^d\left|\vectheta_{\eta}^{\#}|_i\right|\left(\frac{1}{s}\sum_{j=1}^s\left|\vectheta_{\eta}^{\#}|_j\right|\right) \leq \frac{1}{s}\| \vectheta_\eta^{\#1}\|_1\| \vectheta_\eta^{\#2}\|_1\leq \| \vectheta_\eta^{\#1}\|_2\| \vectheta_\eta\|_2.
		\end{align}
		Then, we have 
		\begin{align}
			\|\vectheta_\eta\|_2^2 \leq \|\vectheta_\eta^{\#1}\|_2^2 +\| \vectheta_\eta^{\#2}\|_2^2\leq  \|\vectheta_\eta^{\#1}\|_2\|\vectheta_\eta\|_2 +\| \vectheta_\eta^{\#1}\|_2\| \vectheta_\eta\|_2\Rightarrow 			\|\vectheta_\eta\|_2 \leq  2\|\vectheta_\eta^{\#1}\|_2,
		\end{align}
		and we have
		\begin{align}
				\|\vectheta_\eta\|_2 \leq  2\|\vectheta_\eta^{\#1}\|_2\leq \frac{2}{\kappa_{l}}\|\Sigma^\frac{1}{2}\vectheta_\eta^{\#1}\|_2\leq\frac{2}{\kappa_{l}}\|\Sigma^\frac{1}{2}\vectheta\|_2.
		\end{align}

\subsection{Proof of Lemma \ref{l:mainpre1-2}}
From the same argument of the proof of Lemma \ref{l:mainpre1-1}, we have \eqref{ine:det:lemma:3}.
From \eqref{ine:det:main:0-1}, the first term of the right-hand side of \eqref{ine:det:lemma:3} is evaluated as
\begin{align}
	\label{ine:det:lemma2:1}
	\lambda_o\sqrt{n}\sum_{i=1}^n\frac{1}{n} h\left(\frac{Y_i-\langle \vecZ_i, \vecbeta^*\rangle}{\lambda_o \sqrt{n}}\right) \vecZ_i^\top\vectheta_\eta&\leq r_{a,\Sigma} r_\Sigma \stackrel{(a)}{\leq} \frac{1}{c_{r_1}\sqrt{s}}r_{a,\Sigma}\|\vectheta_\eta\|_1.
\end{align}
where (a) follows from $r_1 = c_{r_1}\sqrt{s}r_\Sigma $.
From~\eqref{ine:det:lemma:3} and \eqref{ine:det:lemma2:1}, we have
\begin{align}
0 \leq  \frac{r_{a,\Sigma}}{c_{r_1}\sqrt{s}}\|\vectheta_\eta\|_1+\eta \lambda_s (\|\vecbeta^*\|_1-\|\hat{\vecbeta}\|_1).
\end{align}
Furthermore, we see
\begin{align}
0
&\leq \frac{r_{a,\Sigma}}{c_{r_1}\sqrt{s}}\|\vectheta_\eta\|_1+\eta \lambda_s(\|\vecbeta^*\|_1-\|\hat{\vecbeta}\|_1) \nonumber \\
& \leq  \frac{ r_{a,\Sigma}}{c_{r_1}\sqrt{s}} (\|\vectheta_{\eta,\mc{J}_{\vecbeta^*}}\|_1+\|\vectheta_{\eta,\mc{J}^c_{\vecbeta^*}}\|_1) + \eta\lambda_s(\|\vecbeta^*_{\mc{J}_{\vecbeta^*}}- \hat{\vecbeta}_{\mc{J}_{\vecbeta^*}}\|_1-\|\hat{\vecbeta}_{\mc{J}^c_{\vecbeta^*}}\|_1)\nonumber\\
& =\left(\lambda_s +  \frac{r_{a,\Sigma}}{c_{r_1}\sqrt{s}}\right)\|\vectheta_{\eta,\mc{J}_{\vecbeta^*}}\|_1
+\left(-\lambda_s+  \frac{r_{a,\Sigma}}{c_{r_1}\sqrt{s}}\right)\|\vectheta_{\eta,\mc{J}^c_{\vecbeta^*}}\|_1.
\end{align}
Then, we have
	\begin{align}
		\label{ine:preshelling}
		\| \vectheta_{\eta,\mc{J}_{\vecbeta^*}^c}\|_1 \leq \frac{\lambda_s +  \frac{ r_{a,\Sigma}}{c_{r_1}\sqrt{s}}}{\lambda_s -  \frac{r_{a,\Sigma}}{c_{r_1}\sqrt{s}} }\| \vectheta_{\eta,\mc{J}_{\vecbeta^*}}\|_1\stackrel{(a)}{\leq} c_{\mr{RE}} \| \vectheta_{\eta,\mc{J}_{\vecbeta^*}}\|_1,
	\end{align}
	where (a) follows from \eqref{ine:det:main:0-3}, and  we have
	\begin{align}
		\|\vectheta_\eta\|_1 = \| \vectheta_{\eta,\mc{J}_{\vecbeta^*}}\|_1 +\| \vectheta_{\eta,\mc{J}_{\vecbeta^*}^c}\|_1 \leq  (1+c_{\mr{RE}})\| \vectheta_{\eta,\mc{J}_{\vecbeta^*}}\|_1\leq (1+c_{\mr{RE}})\sqrt{s}\| \vectheta_{\eta,\mc{J}_{\vecbeta^*}}\|_2.
	\end{align}
	From \eqref{ine:preshelling} and the restricted eigenvalue condition, we have
	\begin{align}
		\|\vectheta_\eta\|_1 \leq  (1+c_{\mr{RE}})\sqrt{s}\| \vectheta_{\eta,\mc{J}_{\vecbeta^*}}\|_2\leq  \frac{1+c_{\mr{RE}}}{\mathfrak{\kappa}}\sqrt{s}\|\Sigma^\frac{1}{2} \vectheta_\eta\|_2.
	\end{align}
	\subsection{Proof of Lemma \ref{l:mainpre2-1}}
	From the same argument of the proof of Lemma \ref{l:mainpre1-1}, we have \eqref{ine:det:lemma:3}.
	From \eqref{ine:det:lemma:3}, we have
	\begin{align}
		\label{ine:det:lemma3:1}
		&\lambda_o\sqrt{n}\sum_{i=1}^n  \frac{1}{n}\left(-h\left(\frac{Y_i-\langle \vecZ_i, \vecbeta^*+\vectheta_\eta\rangle}{\lambda_o \sqrt{n}}\right)+h\left(\frac{Y_i-\langle \vecZ_i, \vecbeta^*\rangle}{\lambda_o \sqrt{n}}\right) \right) \vecZ_i^\top \vectheta_\eta\nonumber\\
		&\leq \lambda_o\sqrt{n}\sum_{i=1}^n\frac{1}{n} h\left(\frac{Y_i-\langle \vecZ_i, \vecbeta^*\rangle}{\lambda_o \sqrt{n}}\right) \vecZ_i^\top\vectheta_\eta+\eta \lambda_s(\|\vecbeta^*\|_1-\|\hat{\vecbeta}\|_1).
	\end{align}
	We evaluate each term of \eqref{ine:det:lemma3:1}. From~\eqref{ine:det:main:0-2} and from  $r_\Sigma = \|\Sigma^\frac{1}{2}\vectheta_\eta\|_2$, the left-hand side  of~\eqref{ine:det:lemma3:1} is evaluated as
	\begin{align}
		\lambda_o\sqrt{n}\sum_{i=1}^n  \frac{1}{n}\left(-h\left(\frac{Y_i-\langle \vecZ_i, \vecbeta^*+\vectheta_\eta\rangle}{\lambda_o \sqrt{n}}\right)+h\left(\frac{Y_i-\langle \vecZ_i, \vecbeta^*\rangle}{\lambda_o \sqrt{n}}\right) \right) \vecZ_i^\top \vectheta_\eta &\geq b \|\Sigma^\frac{1}{2}\vectheta_\eta\|_2^2 -r_{b,\Sigma}\|\Sigma^\frac{1}{2}\vectheta_\eta\|_2.
	\end{align}
	From~\eqref{ine:det:main:0-1} and $r_\Sigma = \|\Sigma^\frac{1}{2}\vectheta_\eta\|_2$, the first term of the right-hand side of \eqref{ine:det:lemma3:1} is evaluated as
	\begin{align}
		\lambda_o\sqrt{n}\sum_{i=1}^n\frac{1}{n} h\left(\frac{Y_i-\langle \vecZ_i, \vecbeta^*\rangle}{\lambda_o \sqrt{n}}\right) \vecZ_i^\top\vectheta_\eta\leq r_{a,\Sigma}\|\Sigma^\frac{1}{2}\vectheta_\eta\|_2.
	\end{align}
	From $r_1 = c_{r_1}\sqrt{s}r_\Sigma  = c_{r_1}\sqrt{s}\|\Sigma^\frac{1}{2}\vectheta_\eta\|_2$, the second term of the right-hand side of \eqref{ine:det:lemma3:1} is evaluated as
	\begin{align}
		\eta \lambda_s(\|\vecbeta^*\|_1-\|\hat{\vecbeta}\|_1) \leq \lambda_s \|\vectheta_\eta\|_1 \leq c_{r_1}\sqrt{s}\lambda_s \|\Sigma^\frac{1}{2} \vectheta_\eta\|_2.
	\end{align}

	Combining the inequalities above, we have
	\begin{align}
		b \|\Sigma^\frac{1}{2}\vectheta_\eta\|_2^2&\leq  \left(r_{a,\Sigma}+r_{b,\Sigma} + c_{r_1}\sqrt{s}\lambda_s\right)\|\Sigma^\frac{1}{2}\vectheta_\eta\|_2,
	 \end{align}
	 and from $\|\Sigma^\frac{1}{2}\vectheta_\eta\|_2\geq 0$, we have
	\begin{align}
		\|\Sigma^\frac{1}{2}\vectheta_\eta\|_2 &\leq \frac{1}{b} \left(r_{a,\Sigma}+r_{b,\Sigma} + c_{r_1}\sqrt{s}\lambda_s\right),
	\end{align}
	and the proof is complete.

\section{Proof of Proposition~\ref{p:main}}
\label{sec:mainproof}
\subsection{Step1}
\label{subsec:mainpropstep1}
We derive a contradiction if $\|\vectheta\|_1> r_1$,  $\|\vectheta\|_2> r_2$ and $\|\Sigma^\frac{1}{2}\vectheta\|_2> r_{\Sigma}$ hold. Assume that  $\|\vectheta\|_1> r_1$,  $\|\vectheta\|_2> r_2$ and $\|\Sigma^\frac{1}{2}\vectheta\|_2> r_{\Sigma}$. Then we can find  $\eta_1,\,\eta_2,\eta_2'\in(0,1)$ such that $\|\vectheta_{\eta_1}\|_1 = r_1$, $\|\vectheta_{\eta_2}\|_2 = r_2$ and $\|\Sigma^\frac{1}{2}\vectheta_{\eta_2'}\|_2 = r_{\Sigma}$ hold.
Define $\eta_3 = \min\{\eta_1,\eta_2, \eta_2'\}$.
We consider the case $\eta_3 = \eta_2'$ in Section \ref{subsubsec:mainpropstep1a},  the case $\eta_3 = \eta_2$ in Section \ref{subsubsec:mainpropstep1b}, and the case $\eta_3 = \eta_1$ in Section \ref{subsubsec:mainpropstep1c}.

\subsubsection{Step 1(a)}
\label{subsubsec:mainpropstep1a}
Assume that $\eta_3 = \eta_2'$. We see that $\|\Sigma^\frac{1}{2}\vectheta_{\eta_3}\|_2 = r_{\Sigma}$, $\|\vectheta_{\eta_3}\|_1 \leq r_1$ and $\|\vectheta_{\eta_3}\|_2 \leq r_2$ hold. Then, from  Lemma~\ref{l:mainpre2-1},  we have 
\begin{align}
\|\Sigma^\frac{1}{2}\vectheta_{\eta_3}\|_2 &\leq \frac{1}{b} \left(r_{a,\Sigma}+r_{b,\Sigma} + c_{r_1}\sqrt{s}\lambda_s\right).
\end{align}
The case $\eta_3=\eta_2'$ is a contradiction from 
$\|\Sigma^\frac{1}{2}\vectheta_{\eta_3}\|_2 = r_\Sigma$
and \eqref{ine:det:main:0-4}.

\subsubsection{Step 1(b)}
\label{subsubsec:mainpropstep1b}
Assume that $\eta_3 = \eta_2$. We see that  $\|\Sigma^\frac{1}{2}\vectheta_{\eta_3}\|_2 \leq r_{\Sigma}$, $\|\vectheta_{\eta_3}\|_1 \leq r_1$ and $\|\vectheta_{\eta_3}\|_2 = r_2$ hold. Then, from  Lemma~\ref{l:mainpre1-1}, we have 
\begin{align}
\|\vectheta_{\eta_3}\|_2 &\leq \frac{3+c_{\mr{RE}}}{\kappa_{l}}\|\Sigma^\frac{1}{2}\vectheta_{\eta_3}\|_2 \leq \frac{3+c_{\mr{RE}}}{\kappa_{l}}r_{\Sigma} \leq  3\frac{1+c_{\mr{RE}}}{\kappa_{l}}r_{\Sigma} .
\end{align}
The case $\eta_3=\eta_2$ is a contradiction from
$\|\vectheta_{\eta_3}\|_2 = r_2 $ and \eqref{ine:det:main:0-4}.

\subsubsection{Step 1(c)}
\label{subsubsec:mainpropstep1c}
Assume that $\eta_3 = \eta_1$. We see that  $\|\Sigma^\frac{1}{2}\vectheta_{\eta_3}\|_2 \leq r_{\Sigma}$, $\|\vectheta_{\eta_3}\|_1 = r_1$ and $\|\vectheta_{\eta_3}\|_2 \leq r_2$ hold. 
Then, from Lemma \ref{l:mainpre1-2}, for $\eta = \eta_{3}$, we have
\begin{align}
\|\vectheta_{\eta_3}\|_1\leq \frac{1+c_{\mr{RE}}}{\kappa}\sqrt{s}\|\Sigma^\frac{1}{2}\vectheta_{\eta_3}\|_2 \leq  \frac{1+c_{\mr{RE}}}{\kappa}\sqrt{s}r_\Sigma \leq 3 \frac{1+c_{\mr{RE}}}{\kappa}\sqrt{s}r_\Sigma .
\end{align}
The case $\eta_3=\eta_1$ is a contradiction from
$\|\vectheta_{\eta_3}\|_1 =r_1$ and  \eqref{ine:det:main:0-4}.
\subsection{Step 2}
\label{asubsec:2}
From the arguments in Section \ref{subsec:mainpropstep1}, we have $\|\Sigma^\frac{1}{2}\vectheta\|_2 \leq r_{\Sigma}$ or $\|\vectheta\|_1 \leq r_1$ or $\|\vectheta\|_2 \leq r_2$  holds.
\begin{itemize}
	\item[(a)]In Section \ref{asubsubsec:2-1}, assume that  $\|\Sigma^\frac{1}{2}\vectheta\|_2 \leq r_{\Sigma}$ and $\|\vectheta\|_1 > r_1$ and $\|\vectheta\|_2 > r_2$
	hold and then derive a contradiction.
 \item[(b)]In Section \ref{asubsubsec:2-2}, assume that  $\|\Sigma^\frac{1}{2}\vectheta\|_2 > r_{\Sigma}$ and $\|\vectheta\|_1 \leq r_1$ and $\|\vectheta\|_2 > r_2$
hold and then derive a contradiction.
\item[(c)]In Section \ref{asubsubsec:2-3}, assume that  $\|\Sigma^\frac{1}{2}\vectheta\|_2 > r_{\Sigma}$ and $\|\vectheta\|_1 >r_1$ and $\|\vectheta\|_2 \leq r_2$
hold and then derive a contradiction.
\item[(d)]In Section \ref{asubsubsec:2-4}, assume that  $\|\Sigma^\frac{1}{2}\vectheta\|_2 > r_{\Sigma}$ and $\|\vectheta\|_1 \leq r_1$ and $\|\vectheta\|_2 \leq r_2$
hold and then derive a contradiction.
\item[(e)]In Section \ref{asubsubsec:2-5}, assume that  $\|\Sigma^\frac{1}{2}\vectheta\|_2 \leq r_{\Sigma}$ and $\|\vectheta\|_1 > r_1$ and $\|\vectheta\|_2 \leq r_2$
hold and then derive a contradiction.
\item[(f)]In Section \ref{asubsubsec:2-6}, assume that  $\|\Sigma^\frac{1}{2}\vectheta\|_2 \leq r_{\Sigma}$ and $\|\vectheta\|_1 \leq r_1$ and $\|\vectheta\|_2 > r_2$
hold and then derive a contradiction.
\end{itemize}
Finally, we have
\begin{align}
	\|\Sigma^\frac{1}{2}(\hat{\vecbeta}-\vecbeta^*)\|_2 \leq r_\Sigma \text {, }\|\hat{\vecbeta}-\vecbeta^*\|_2 \leq r_2 \text {, and }	\|\hat{\vecbeta}-\vecbeta^*\|_1 \leq r_1,
\end{align}
and the proof is complete.
\subsubsection{Step 2(a)}
\label{asubsubsec:2-1}
Assume that $\|\Sigma^\frac{1}{2}\vectheta\|_2 \leq r_{\Sigma}$ and $\|\vectheta\|_1 > r_1$ and $\|\vectheta\|_2 > r_2$ hold, and
then we can find $\eta_4, \eta_4'\in (0,1)$ such that $\|\vectheta_{\eta_4}\|_1 = r_1$ and $\|\vectheta_{\eta_4'}\|_2 = r_2$ hold.
We note that $\|\Sigma^\frac{1}{2}\vectheta_{\eta_4}\|_2\leq r_\Sigma$ and $\|\Sigma^\frac{1}{2}\vectheta_{\eta_4'}\|_2\leq r_\Sigma$ also hold.
Then, from the same arguments of Sections \ref{subsubsec:mainpropstep1b} and \ref{subsubsec:mainpropstep1c}, we have a contradiction.

\subsubsection{Step 2(b)}
\label{asubsubsec:2-2}
Assume that $\|\Sigma^\frac{1}{2}\vectheta\|_2 > r_{\Sigma}$ and $\|\vectheta\|_1 \leq r_1$ and $\|\vectheta\|_2 > r_2$ hold, and
then we can find $\eta_5, \eta_5'\in (0,1)$ such that $\|\Sigma^\frac{1}{2}\vectheta_{\eta_5}\|_2 = r_{\Sigma}$ and $\|\vectheta_{\eta_5'}\|_2 = r_2$ hold.
We note that $\|\vectheta_{\eta_5}\|_1\leq r_1$ and $\|\vectheta_{\eta_5'}\|_1\leq r_1$  also hold.
Then, from the same arguments of Sections \ref{subsubsec:mainpropstep1a} and  \ref{subsubsec:mainpropstep1b}, we have a contradiction.

\subsubsection{Step 2(c)}
\label{asubsubsec:2-3}
Assume that $\|\Sigma^\frac{1}{2}\vectheta\|_2 > r_{\Sigma}$ and $\|\vectheta\|_1 >r_1$ and $\|\vectheta\|_2 \leq r_2$ hold and,
then we can find $\eta_6,\eta_6'\in (0,1)$ such that $\|\vectheta_{\eta_6}\|_1=r_1$ and  $\|\Sigma^\frac{1}{2}\vectheta_{\eta_6'}\|_2 =r_\Sigma$ hold.
We note that $\|\vectheta_{\eta_6}\|_2\leq r_2$ and $\|\vectheta_{\eta_6'}\|_2\leq r_2$  also hold.
Then, from the same arguments of Sections \ref{subsubsec:mainpropstep1a} and  \ref{subsubsec:mainpropstep1c}, we have a contradiction.

\subsubsection{Step 2(d)}
\label{asubsubsec:2-4}
Assume that $\|\Sigma^\frac{1}{2}\vectheta\|_2 > r_{\Sigma}$ and $\|\vectheta\|_1 \leq r_1$ and $\|\vectheta\|_2 \leq r_2$ hold and,
then we can find $\eta_7\in (0,1)$ such that  $\|\Sigma^\frac{1}{2}\vectheta_{\eta_7}\|_2 = r_{\Sigma}$ holds.
We note that $\|\vectheta_{\eta_7}\|_1\leq r_1$ and  $\|\vectheta_{\eta_7}\|_2\leq r_2$ also hold.
Then, from the same arguments of Section \ref{subsubsec:mainpropstep1a}, we have  a contradiction.

\subsubsection{Step 2(e)}
\label{asubsubsec:2-5}
Assume that $\|\Sigma^\frac{1}{2}\vectheta\|_2 \leq r_{\Sigma}$ and $\|\vectheta\|_1 > r_1$ and $\|\vectheta\|_2 \leq r_2$ hold and,
then we can find $\eta_8\in (0,1)$ such that $\|\vectheta_{\eta_8}\|_1=r_1$ holds.
We note that $\|\Sigma^\frac{1}{2}\vectheta_{\eta_8}\|_2\leq r_\Sigma$ and $\|\vectheta_{\eta_8}\|_2\leq r_2$ also hold.
Then, from the same arguments of Section \ref{subsubsec:mainpropstep1c}, we have  a contradiction.

\subsubsection{Step 2(f)}
\label{asubsubsec:2-6}
Assume that $\|\Sigma^\frac{1}{2}\vectheta\|_2 \leq r_{\Sigma}$ and $\|\vectheta\|_1 \leq r_1$ and $\|\vectheta\|_2 > r_2$hold and,
then we can find $\eta_9\in (0,1)$ such that $\|\vectheta_{\eta_9}\|_2=r_2$ holds.
We note that $\|\Sigma^\frac{1}{2}\vectheta_{\eta_9}\|_2\leq r_\Sigma$  and $\|\vectheta_{\eta_9}\|_1\leq r_1$ also hold.
Then, from the same arguments of Section \ref{subsubsec:mainpropstep1b}, we have  a contradiction.

\section{Proof of Proposition \ref{p:main1:no}}
\label{sec:proof:p:main1:no}
\begin{proof}
	From simple algebra, we have
	\begin{align}
		\label{mcssub2}
		&\sup_{\vecv  \in r_1 \mbb{B}^d_1 \cap r_2 \mbb{B}^d_2 \cap r_\Sigma \mbb{B}^d_\Sigma}\sum_{i=1}^n \frac{1}{n} h\left(\frac{y_i-\langle \tilde{\vecx}_i, \vecbeta^*\rangle}{\lambda_o \sqrt{n}}\right) \langle \tilde{\vecx}_i,\vecv\rangle\nonumber\\
		&=\sup_{\vecv  \in r_1 \mbb{B}^d_1 \cap r_2 \mbb{B}^d_2 \cap r_\Sigma \mbb{B}^d_\Sigma}\left(\sum_{i=1}^n \frac{1}{n} h\left(\frac{\langle \vecx_i-\tilde{\vecx}_i, \vecbeta^*\rangle+\xi_i}{\lambda_o \sqrt{n}}\right) \langle \tilde{\vecx}_i,\vecv\rangle \right)\nonumber\\
		&\quad \quad\leq\underbrace{\sup_{\vecv  \in r_1 \mbb{B}^d_1 \cap r_2 \mbb{B}^d_2 \cap r_\Sigma \mbb{B}^d_\Sigma}\sum_{i=1}^n \frac{1}{n}h\left(\frac{\langle \vecx_i-\tilde{\vecx}_i, \vecbeta^*\rangle+\xi_i}{\lambda_o \sqrt{n}}\right) \langle \tilde{\vecx}_i,\vecv\rangle -\mbb{E}  h\left(\frac{\langle \vecx_i-\tilde{\vecx}_i, \vecbeta^*\rangle+\xi_i}{\lambda_o \sqrt{n}}\right)  \langle \tilde{\vecx}_i,\vecv\rangle }_{T_1}\nonumber \\
		&\quad \quad\quad \quad+\underbrace{\sup_{\vecv  \in r_1 \mbb{B}^d_1 \cap r_2 \mbb{B}^d_2 \cap r_\Sigma \mbb{B}^d_\Sigma}\left\{\mbb{E}   h\left(\frac{\langle \vecx_i-\tilde{\vecx}_i, \vecbeta^*\rangle+\xi_i}{\lambda_o \sqrt{n}}\right) \langle \tilde{\vecx}_i,\vecv\rangle-\mbb{E}   h\left(\frac{\langle \vecx_i-\tilde{\vecx}_i, \vecbeta^*\rangle+\xi_i}{\lambda_o \sqrt{n}}\right) \langle \vecx_i,\vecv\rangle \right\}}_{T_2}\nonumber \\
		&\quad \quad\quad \quad+\underbrace{\sup_{\vecv  \in r_1 \mbb{B}^d_1 \cap r_2 \mbb{B}^d_2 \cap r_\Sigma \mbb{B}^d_\Sigma}\left\{\mbb{E}   h\left(\frac{\langle \vecx_i-\tilde{\vecx}_i, \vecbeta^*\rangle+\xi_i}{\lambda_o \sqrt{n}}\right) \langle \vecx_i,\vecv\rangle-\mbb{E}   h\left(\frac{\xi_i}{\lambda_o\sqrt{n}}\right)\langle \vecx_i,\vecv\rangle\right\}}_{T_3}.
	\end{align}

	First, we evaluate $T_1$. 
	From union bound, we have 
	\begin{align}
		T_1\leq \max_{j=1,\cdots,d} \left|\sum_{i=1}^n \frac{1}{n} h\left(\frac{\langle \vecx_i-\tilde{\vecx}_i, \vecbeta^*\rangle+\xi_i}{\lambda_o \sqrt{n}}\right) \tilde{\vecx}_{i_j} -\mbb{E}\sum_{i=1}^n \frac{1}{n} h\left(\frac{\langle \vecx_i-\tilde{\vecx}_i, \vecbeta^*\rangle+\xi_i}{\lambda_o \sqrt{n}}\right) \tilde{\vecx}_{i_j} \right|r_1.
	\end{align}
	We note that, for any $1\leq j\leq d$,
	\begin{align}
		&\mbb{E}\left(h\left(\frac{\langle \vecx_i-\tilde{\vecx}_i, \vecbeta^*\rangle+\xi_i}{\lambda_o \sqrt{n}}\right) \tilde{\vecx}_{i_j}-\mbb{E}h\left(\frac{\langle \vecx_i-\tilde{\vecx}_i, \vecbeta^*\rangle+\xi_i}{\lambda_o \sqrt{n}}\right) \tilde{\vecx}_{i_j}\right)^2\leq 	\mbb{E}h\left(\frac{\langle \vecx_i-\tilde{\vecx}_i, \vecbeta^*\rangle+\xi_i}{\lambda_o \sqrt{n}}\right)^2 \tilde{\vecx}_{i_j}^2\leq \mbb{E}\tilde{\vecx}_{i_j}^2\leq 1, \nonumber\\
		&\mbb{E}\left(h\left(\frac{\langle \vecx_i-\tilde{\vecx}_i, \vecbeta^*\rangle+\xi_i}{\lambda_o \sqrt{n}}\right) \tilde{\vecx}_{i_j}-\mbb{E}h\left(\frac{\langle \vecx_i-\tilde{\vecx}_i, \vecbeta^*\rangle+\xi_i}{\lambda_o \sqrt{n}}\right) \tilde{\vecx}_{i_j}\right)^p \nonumber\\
		&\leq  \left(2 \tau_\vecx\right)^{p-2}\mbb{E}\left(h\left(\frac{\langle \vecx_i-\tilde{\vecx}_i, \vecbeta^*\rangle+\xi_i}{\lambda_o \sqrt{n}}\right) \tilde{\vecx}_{i_j}-\mbb{E}h\left(\frac{\langle \vecx_i-\tilde{\vecx}_i, \vecbeta^*\rangle+\xi_i}{\lambda_o \sqrt{n}}\right) \tilde{\vecx}_{i_j}\right)^2  \leq \left(2 \tau_\vecx\right)^{p-2},
	\end{align}
	and from Bernstein's inequality (Lemma 5.1 of \cite{Dir2015Tail}), we have
	\begin{align}
		 \mbb{P}\left(\left|\sum_{i=1}^n \frac{1}{n} h\left(\frac{\langle \vecx_i-\tilde{\vecx}_i, \vecbeta^*\rangle+\xi_i}{\lambda_o \sqrt{n}}\right) \tilde{\vecx}_{i_j} -\mbb{E}\sum_{i=1}^n \frac{1}{n} h\left(\frac{\langle \vecx_i-\tilde{\vecx}_i, \vecbeta^*\rangle+\xi_i}{\lambda_o \sqrt{n}}\right) \tilde{\vecx}_{i_j}\right|\geq \sqrt{2\frac{t}{n}}+2\frac{\tau_{\vecx}t}{n}\right)\leq e^{-t},
	\end{align}
	and, with probability at least $1-\delta$, we have 
	\begin{align}
		T_1\leq \left(\sqrt{2\frac{\log(d/\delta)}{n}}+2\frac{\tau_{\vecx}\log(d/\delta)}{n}\right)r_1.
	\end{align}
	
	Second, we evaluate $T_2$.
	\begin{align}
		\label{ine:T_3}
		T_2
		&\stackrel{(a)}{\leq}   \left\|\mbb{E}  h\left(\frac{\langle \vecx_i-\tilde{\vecx}_i, \vecbeta^*\rangle+\xi_i}{\lambda_o \sqrt{n}}\right) (\tilde{\vecx}_i-\vecx_i)\right\|_\infty r_1\nonumber \\
		&=  \max_{j \in \{1,\cdots,d\}} \left|\mbb{E}  h\left(\frac{\langle \vecx_i-\tilde{\vecx}_i, \vecbeta^*\rangle+\xi_i}{\lambda_o \sqrt{n}}\right)(\tilde{\vecx}_{i_j}-\vecx_{i_j}) \right|r_1\nonumber \\
		&\stackrel{(b)}{\leq}   \max_{j \in \{1,\cdots,d\}} \mbb{E} \left|\tilde{\vecx}_{i_j}-\vecx_{i_j}\right| r_1\nonumber \\
		&=   \max_{j \in \{1,\cdots,d\}} \mbb{E} \left|\tilde{\vecx}_{i_j}-\vecx_{i_j}\right|\mr{I}_{|\vecx_{i_j} |\geq \tau_{\vecx}}  r_1\nonumber \\
		&\leq    \max_{j \in \{1,\cdots,d\}} 2\mbb{E} \left|\vecx_{i_j}\right|\mr{I}_{|\vecx_{i_j} |\geq \tau_{\vecx}}  r_1\nonumber \\
		&\stackrel{(b)}{\leq}  2 \max_{j \in \{1,\cdots,d\}}(  \mbb{E}\vecx_{i_j}^4)^\frac{1}{4}(\mbb{E}\mr{I}_{|\vecx_{i_j} |\geq \tau_{\vecx}})^{\frac{3}{4}}r_1\stackrel{(c)}{\leq}  2 \frac{K^4}{\tau_{\vecx}^3}r_1,
	\end{align}
	where (a) follows from H{\"o}lder's inequality, (b) follows from $-1\leq h(\cdot)\leq 1$, and (c) follows from Lemma \ref{a:thex}. 

	Lastly, we evaluate $T_3$. Define  $h'(\cdot)$ as the differential of $h(\cdot)$. Let $0<\upsilon <1$. We have
	\begin{align}
		\label{ine:T_4}
		T_3
		&\stackrel{(a)}{\leq }\sup_{\vecv  \in r_1 \mbb{B}^d_1 \cap r_2 \mbb{B}^d_2 \cap r_\Sigma \mbb{B}^d_\Sigma}(\mbb{E}\langle \vecx_i,\vecv\rangle^4)^\frac{1}{4}\left( \mbb{E}\left|h\left(\frac{\langle \vecx_i-\tilde{\vecx}_i, \vecbeta^*\rangle+\xi_i}{\lambda_o \sqrt{n}}\right) -h\left(\frac{\xi_i}{\lambda_o\sqrt{n}}\right)\right|^\frac{4}{3}\right)^\frac{3}{4}\nonumber \\
		&\stackrel{(b)}{\leq}Kr_\Sigma \left( \mbb{E}\left|h\left(\frac{\langle \vecx_i-\tilde{\vecx}_i, \vecbeta^*\rangle+\xi_i}{\lambda_o \sqrt{n}}\right) -h\left(\frac{\xi_i}{\lambda_o\sqrt{n}}\right)\right|^\upsilon\left|h\left(\frac{\langle \vecx_i-\tilde{\vecx}_i, \vecbeta^*\rangle+\xi_i}{\lambda_o \sqrt{n}}\right) -h\left(\frac{\xi_i}{\lambda_o\sqrt{n}}\right)\right|^{\frac{4}{3}-\upsilon} \right)^\frac{3}{4}\nonumber \\
		&\stackrel{(c)}{\leq}  2^{1-\frac{3\upsilon}{4}}Kr_\Sigma \left( \mbb{E}\left|h\left(\frac{\langle \vecx_i-\tilde{\vecx}_i, \vecbeta^*\rangle+\xi_i}{\lambda_o \sqrt{n}}\right) -h\left(\frac{\xi_i}{\lambda_o\sqrt{n}}\right)\right|^\upsilon\right)^\frac{3}{4}\nonumber \\
		&\stackrel{(d)}{\leq }   2^{1-\frac{3\upsilon}{4}} Kr_\Sigma\left( \mbb{E}\left|\frac{(\vecx_{i}-\tilde{\vecx}_i)^\top \vecbeta^*}{\lambda_o\sqrt{n}}\right|^\upsilon \right)^\frac{3}{4}\nonumber\\
		&\stackrel{(e)}{\leq} 2^{1-\frac{3\upsilon}{4}} K \left( \mbb{E}\left|\langle \vecx_{i}-\tilde{\vecx}_i, \vecbeta^*\rangle \right|^\upsilon \right)^\frac{3}{4} r_\Sigma,
	\end{align}
	where (a) follows from H{\"o}lder's inequality, (b)  follows from the finite kurtosis property of $\vecx_i$,  (c) follows from the fact that $|h(\cdot)|\leq 1$,  (d) follows from Lipschitz continuity of $h(\cdot)$ and  from the fact $|h'(\cdot)|\leq 1$, and (e) follows from $\lambda_o\sqrt{n}\geq 1$.

	We compute $\left( \mbb{E}\left|\langle \vecx_{i}-\tilde{\vecx}_i, \vecbeta^*\rangle \right|^\upsilon \right)^\frac{3}{4}$. Define $\mr{dom}(\vecbeta^*)$ as the set of the indices such that $\vecbeta^*_j\neq 0$.  We have
\begin{align}
	\left( \mbb{E}\left|\langle \vecx_{i}-\tilde{\vecx}_i, \vecbeta^*\rangle \right|^\upsilon \right)^\frac{3}{4}&\leq \left\{ \mbb{E} \left( \sum_{j\in \mr{dom}(\vecbeta^*) }|\vecbeta^*_j (\vecx_{i_j}-\tilde{\vecx}_{i_j} )| \right)^\upsilon \right\}^\frac{3}{4}\nonumber \\
	&\stackrel{(a)}{\leq}\left(  \sum_{j\in \mr{dom}(\vecbeta^*) }\mbb{E} \left|\vecbeta^*_j (\vecx_{i_j}-\tilde{\vecx}_{i_j} )\right|^\upsilon \right)^\frac{3}{4}\nonumber \\
&\leq c_{\vecbeta}^\frac{3\upsilon}{4} s^\frac{3}{4} \left( \mbb{E} \left|\vecx_{i_j}-\tilde{\vecx}_{i_j}\right|^\upsilon  \right)^\frac{3}{4}\nonumber \\
&\leq  2^\frac{3 \upsilon}{4}  c_{\vecbeta}^\frac{3\upsilon}{4} s^\frac{3}{4} \left( \mbb{E} \left|\mr{I}_{|\vecx_{i_j}|\geq \tau_\vecx}\vecx_{i_j} \right|^\upsilon \right)^\frac{3}{4}\nonumber \\
&\stackrel{(b)}{\leq} 2^\frac{3\upsilon}{4} c_{\vecbeta}^\frac{3\upsilon}{4}   s^\frac{3}{4}\left(\mbb{E}\vecx_{i_j} ^4\right)^{ \frac{\upsilon}{4}\times\frac{3}{4}}\left(\mbb{E}\mr{I}_{|\vecx_{i_j}|\geq \tau_{{\vecx} }} \right)^{ \frac{4-\upsilon}{4}\times\frac{3}{4}}\nonumber \\
&\stackrel{(c)}{\leq} 2^\frac{3\upsilon}{4}  c_{\vecbeta}^\frac{3\upsilon}{4}  K^3 s^\frac{3}{4}\left(\frac{1}{\tau_{\vecx}}\right)^{ 3-\frac{3\upsilon}{4}},
\end{align}
where (a) follows from the subadditivity,  (b) follows  H{\"o}lder's inequality, and (c) follows from Lemma \ref{a:thex}.

Set $\upsilon = 1/3$. Combining the arguments above, with probability at least $1-\delta$, we have
	\begin{align}
		&\sup_{\vecv  \in r_1 \mbb{B}^d_1 \cap r_2 \mbb{B}^d_2 \cap r_\Sigma \mbb{B}^d_\Sigma}\left|\sum_{i=1}^n \frac{1}{n} h\left(\frac{\langle \vecx_i-\tilde{\vecx}_i, \vecbeta^*\rangle+\xi_i}{\lambda_o \sqrt{n}}\right) \langle \tilde{\vecx}_i,\vecv\rangle \right|\nonumber \\
		&\leq  \left(\sqrt{2\frac{\log(d/\delta)}{n}}+2\frac{\tau_{\vecx}\log(d/\delta)}{n}\right)r_1 +2 \frac{K^4}{\tau_{\vecx}^3}r_1+  4K^4 c_{\vecbeta}^\frac{1}{4}  s^\frac{3}{4}\left(\frac{1}{\tau_{\vecx}}\right)^{ 3-\frac{3\upsilon}{4}}r_\Sigma \nonumber\\
		&\stackrel{(a)}{\leq} 4 \left(c_{r_1}\sqrt{s\frac{\log (d/\delta)}{n}}+c_{r_1}\tau_\vecx\sqrt{s}\frac{\log (d/\delta)}{n}+K^4c_{r_1}\frac{\sqrt{s}}{\tau_{\vecx}^3}+ K^4 c_{\vecbeta}^\frac{1}{4}s^\frac{3}{4}\left(\frac{1}{\tau_{\vecx}}\right)^{ 3-\frac{1}{4}}\right)r_\Sigma,
	\end{align}
	where (a) follows from $c_{r_1}\sqrt{s}r_\Sigma = r_1$.
\end{proof}

\subsection{Proof of Proposition \ref{p:main:sc:no}}
\label{sec:proof:p:main:sc:no}
\begin{proof}
	Define 
	\begin{align}
		\mathfrak{V}^\Sigma_{r_1,r_2,r_\Sigma} = \{\vecv \in\mbb{R}^d \mid \vecv \in r_1 \mbb{B}^d_1 \cap r_2 \mbb{B}^d_2,\, \|\Sigma^\frac{1}{2}\vecv\|_2 = r_\Sigma\}.
	\end{align}
	This proposition is proved in a manner similar to  the proof of Proposition B.1 of \cite{CheZho2020Robust}.
 	The L.H.S of \eqref{ine:sc:no} divided by $\lambda_o^2$ can be expressed as
	\begin{align}
		 \sum_{i=1}^n  \left(-h\left(\frac{\langle \vecx_i-\tilde{\vecx}_i, \vecbeta^*\rangle+\xi_i}{\lambda_o \sqrt{n}}-\frac{\tilde{\vecx}_i^\top \vecv}{\lambda_o\sqrt{n}}\right)+h\left(\frac{\langle \vecx_i-\tilde{\vecx}_i, \vecbeta^*\rangle+\xi_i}{\lambda_o \sqrt{n}}\right)  \right)\frac{\tilde{\vecx}_i^\top \vecv}{\lambda_o\sqrt{n}}.
	\end{align}
	From the convexity of the Huber loss,  we have
	\begin{align}
		&\sum_{i=1}^n  \left(-h\left(\frac{\langle \vecx_i-\tilde{\vecx}_i, \vecbeta^*\rangle+\xi_i}{\lambda_o \sqrt{n}}-\frac{\tilde{\vecx}_i^\top \vecv}{\lambda_o\sqrt{n}}\right)+h\left(\frac{\langle \vecx_i-\tilde{\vecx}_i, \vecbeta^*\rangle+\xi_i}{\lambda_o \sqrt{n}}\right)  \right)\frac{\tilde{\vecx}_i^\top \vecv}{\lambda_o\sqrt{n}} \\
		&\geq \sum_{i=1}^n  \left(-h\left(\frac{\langle \vecx_i-\tilde{\vecx}_i, \vecbeta^*\rangle+\xi_i}{\lambda_o \sqrt{n}}-\frac{\tilde{\vecx}_i^\top \vecv}{\lambda_o\sqrt{n}}\right)+h\left(\frac{\langle \vecx_i-\tilde{\vecx}_i, \vecbeta^*\rangle+\xi_i}{\lambda_o \sqrt{n}}\right)  \right)\frac{\tilde{\vecx}_i^\top \vecv}{\lambda_o\sqrt{n}}\mr{I}_{E_i}.
	\end{align}
	Define the functions
	\begin{align}
	\label{def:phipsi}
		\varphi(x) =\begin{cases}
		x^2   &  \mbox{ if }  |x| \leq  1/4\\
		(x-1/4)^2   &  \mbox{ if }  1/4\leq x  \leq  1/2 \\
		(x+1/4)^2   &  \mbox{ if }  -1/2\leq x  \leq -1/4   \\
		0 & \mbox{ if } |x| >1/2
	\end{cases} ~\mbox{ and }~
		\psi(x) = \mr{I}_{(|x| \leq 1/2 ) }.
	\end{align}
	Let $f_i(\vecv) = \varphi\left(\frac{\tilde{\vecx}_i^\top \vecv}{\lambda_o\sqrt{n}}\right) \psi\left(\frac{\langle \vecx_i-\tilde{\vecx}_i, \vecbeta^*\rangle+\xi_i}{\lambda_o \sqrt{n}}\right)$
	and we have
	\begin{align}
	\label{ine:huv-conv-f}
		&\sum_{i=1}^n  \left(-h\left(\frac{\langle \vecx_i-\tilde{\vecx}_i, \vecbeta^*\rangle+\xi_i}{\lambda_o \sqrt{n}}-\frac{\tilde{\vecx}_i^\top \vecv}{\lambda_o\sqrt{n}}\right)+h\left(\frac{\langle \vecx_i-\tilde{\vecx}_i, \vecbeta^*\rangle+\xi_i}{\lambda_o \sqrt{n}}\right)  \right)\frac{\tilde{\vecx}_i^\top \vecv}{\lambda_o\sqrt{n}}\nonumber\\
		 &\geq \sum_{i=1}^n  \left(\frac{\tilde{\vecx}_i^\top \vecv}{\lambda_o\sqrt{n}}\right)^2\mr{I}_{E_i}\stackrel{(a)}{\geq} \sum_{i=1}^n  \varphi\left(\frac{\tilde{\vecx}_i^\top \vecv}{\lambda_o\sqrt{n}}\right) \psi\left(\frac{\langle \vecx_i-\tilde{\vecx}_i, \vecbeta^*\rangle+\xi_i}{\lambda_o \sqrt{n}}\right)=\sum_{i=1}^n f_i(\vecv),
	\end{align}
	where (a) follows from $\varphi(v) \geq v^2$ for $|v| \leq 1/2$.  We note that 
	\begin{align}
	\label{ine:f-1/4}
		f_i(\vecv) \leq\varphi(v_i) \leq \min\left\{\left(\frac{\tilde{\vecx}_i^\top \vecv}{\lambda_o\sqrt{n}}\right)^2,\frac{1}{4}\right\}.
	\end{align}
	To bound $\sum_{i=1}^n f_i(\vecv)$ from below, for any fixed $ \vecv \in  \mathfrak{V}^\Sigma_{r_1,r_2,r_\Sigma}$, we have
	\begin{align}
	\label{ine:fbelow}
		\sum_{i=1}^n f_i(\vecv)&\geq \mbb{E}f(\vecv) -\sup_{\vecv \in \mathfrak{V}^\Sigma_{r_1,r_2,r_\Sigma}}  \Big|\sum_{i=1}^n f_i(\vecv)-\mbb{E}\sum_{i=1}^n f_i(\vecv)\Big|.
	\end{align}
	Define the supremum of a random process indexed by $\mathfrak{V}^\Sigma_{r_1,r_2,r_\Sigma}$:
	\begin{align}
	\label{ap:delta}
		\Delta  :=  \sup_{ \vecv \in \mathfrak{V}^\Sigma_{r_1,r_2,r_\Sigma}} \left| \sum_{i=1}^n f_i(\vecv) - \mbb{E}\sum_{i=1}^n f_i	(\vecv) \right| .  
	\end{align}
	From \eqref{ine:huv-conv-f} and \eqref{def:phipsi}, we have
	\begin{align}
	\label{ine:aplower:tmp}
		\mbb{E}\sum_{i=1}^n f_i(\vecv)&\geq \sum_{i=1}^n\mbb{E} \left(\frac{\tilde{\vecx}_i^\top \vecv}{\lambda_o\sqrt{n}}\right)^2- \sum_{i=1}^n\mbb{E}\left(\frac{\tilde{\vecx}_i^\top \vecv}{\lambda_o\sqrt{n}}\right)^2 \mr{I}_{\left|\frac{\tilde{\vecx}_i^\top \vecv}{\lambda_o\sqrt{n}} \right|\geq 1/2}-  \sum_{i=1}^n\mbb{E}\left(\frac{\tilde{\vecx}_i^\top \vecv}{\lambda_o\sqrt{n}}\right)^2 \mr{I}_{ \left|\frac{\langle \tilde{\vecx}_i-\tilde{\vecx}_i, \vecbeta^*\rangle+\xi_i}{\lambda_o \sqrt{n}}\right|\geq 1/2 }.
	\end{align}
	For $\mbb{E} \left(\frac{\tilde{\vecx}_i^\top \vecv}{\lambda_o\sqrt{n}}\right)^2$, from Lemma \ref{a:l:1}, we have
	\begin{align}
		\label{ine:v2}
	 \mbb{E}(\tilde{\vecx}_i^\top \vecv)^2  \geq 		-4\frac{K^4} {\tau_{\vecx}^2}\|\vecv\|_1^2+ \mbb{E}(\vecx_i^\top \vecv)^2  &=		-2\frac{K^4} {\tau_{\vecx}^2}\|\vecv\|_1^2+  \|\Sigma^\frac{1}{2}\vecv\|_2^2\nonumber\\
	 &\geq -4\frac{K^4} {\tau_{\vecx}^2}r_1^2+  \|\Sigma^\frac{1}{2}\vecv\|_2^2\nonumber\\
	 &\stackrel{(a)}{\geq} -4K^4c_{r_1}^2 \frac{s} {\tau_{\vecx}^2}r_\Sigma^2+  \|\Sigma^\frac{1}{2}\vecv\|_2^2\nonumber\\
	 &\stackrel{(b)}{\geq} \left(-4K^4c_{r_1}^2 \frac{s} {\tau_{\vecx}^2}+  1\right) \|\Sigma^\frac{1}{2}\vecv\|_2^2,
	\end{align}
	where (a) follows from $r_1 = c_{r_1}\sqrt{s}r_\Sigma$ and (b) follows from $r_\Sigma =  \|\Sigma^\frac{1}{2}\vecv\|_2^2$.

For $\mbb{E}\left(\frac{\tilde{\vecx}_i^\top \vecv}{\lambda_o\sqrt{n}}\right)^2 \mr{I}_{\left|\frac{\tilde{\vecx}_i^\top \vecv}{\lambda_o\sqrt{n}} \right|\geq 1/2}$, we have
	\begin{align}
		\label{ine:q2}
		\mbb{E}(\tilde{\vecx}_i^\top \vecv)^2\mr{I}_{\left|\frac{\tilde{\vecx}_i^\top \vecv}{\lambda_o\sqrt{n}} \right|\geq 1/2 }& = \mbb{E} \vecv^\top(\tilde{\vecx}_i \tilde{\vecx}_i^\top -\vecx_i \vecx_i^\top +\vecx_i \vecx_i^\top )\vecv \mr{I}_{\left|\frac{\tilde{\vecx}_i^\top \vecv}{\lambda_o\sqrt{n}} \right|\geq 1/2}\nonumber\\
		&\stackrel{(a)}{\leq} 4\frac{K^4} {\tau_{\vecx}^2}\|\vecv\|_1^2+\mbb{E}(\vecx_i^\top \vecv)^2\mr{I}_{\left|\frac{\tilde{\vecx}_i^\top \vecv}{\lambda_o\sqrt{n}} \right|\geq 1/2}\nonumber\\
		&\stackrel{(b)}{\leq} 4\frac{K^4} {\tau_{\vecx}^2}\|\vecv\|_1^2+\sqrt{\mbb{E}(\vecx_i^\top \vecv)^4}\sqrt{\mbb{E}\mr{I}_{\left|\frac{\tilde{\vecx}_i^\top \vecv}{\lambda_o\sqrt{n}} \right|\geq 1/2}}\nonumber\\
		&\stackrel{(c)}{\leq} 4\frac{K^4} {\tau_{\vecx}^2}\|\vecv\|_1^2+\sqrt{2}K^2\|\Sigma^\frac{1}{2}\vecv\|_2^2\sqrt{\frac{\mbb{E}(|\langle \vecx_i^\top \vecv \rangle |+|\langle \vecx_i^\top  - \tilde{\vecx}_i , \vecv \rangle |)}{\lambda_o\sqrt{n}}}\nonumber\\
		&\stackrel{(d)}{\leq} 4\frac{K^4} {\tau_{\vecx}^2}\|\vecv\|_1^2+K^2\|\Sigma^\frac{1}{2}\vecv\|_2^2\sqrt{\frac{2}{\lambda_o\sqrt{n}}}\sqrt{r_\Sigma+\|\vecv\|_1 \frac{K^4}{\tau_{\vecx}^3}}\nonumber\\
		&\stackrel{(e)}{\leq} \left(4K^4 c_{r_1}^2\frac{s} {\tau_{\vecx}^2}+ K^2\sqrt{\frac{2}{\lambda_o\sqrt{n}}}\sqrt{1+c_{r_1}\sqrt{s}\frac{K^4}{\tau_{\vecx}^3}}\right)\|\Sigma^\frac{1}{2}\vecv\|_2^2,
	\end{align}
	where (a) follows from \eqref{eq1:a:l:1} in the proof of Lemma \ref{a:l:1}, (b) follows from H{\"o}lder's inequality,  (c) follows from the finite kurtosis property of $\vecx_i$, the relationship of expectation of indicator function and probability, and Markov's inequality, (d) follows from Lemma \ref{a:l:1} and the definition of $\vecx_i$, and (e) follows from $r_1 = c_{r_1}\sqrt{s}r_\Sigma$, $r_\Sigma \leq 1$ and $r_\Sigma = \|\Sigma^\frac{1}{2}\vecv\|_2^2$.
	
For $\mbb{E}\left(\frac{\tilde{\vecx}_i^\top \vecv}{\lambda_o\sqrt{n}}\right)^2 \mr{I}_{ \left|\frac{\langle \vecx_i-\tilde{\vecx}_i, \vecbeta^*\rangle+\xi_i}{\lambda_o \sqrt{n}}\right|\geq 1/2 }$,we have
	\begin{align}
		\label{ine:q3}
		\mbb{E}(\tilde{\vecx}_i^\top \vecv)^2\mr{I}_{ \left|\frac{\langle \vecx_i-\tilde{\vecx}_i, \vecbeta^*\rangle+\xi_i}{\lambda_o \sqrt{n}}\right|\geq 1/2}& = \mbb{E} \vecv^\top(\tilde{\vecx}_i \tilde{\vecx}_i^\top -\vecx_i \vecx_i^\top +\vecx_i \vecx_i^\top )\vecv \mr{I}_{ \left|\frac{\langle \vecx_i-\tilde{\vecx}_i, \vecbeta^*\rangle+\xi_i}{\lambda_o \sqrt{n}}\right|\geq 1/2}\nonumber\\
		&\stackrel{(a)}{\leq}  4\frac{K^4} {\tau_{\vecx}^2}\|\vecv\|_1^2+\mbb{E}(\vecx_i^\top \vecv)^2\mr{I}_{ \left|\frac{\langle \vecx_i-\tilde{\vecx}_i, \vecbeta^*\rangle+\xi_i}{\lambda_o \sqrt{n}}\right|\geq 1/2}\nonumber\\
		&\stackrel{(b)}{\leq}  4\frac{K^4} {\tau_{\vecx}^2}\|\vecv\|_1^2+\sqrt{\mbb{E}(\vecx_i^\top \vecv)^4}\sqrt{\mbb{E}\mr{I}_{ \left|\frac{\langle \vecx_i-\tilde{\vecx}_i, \vecbeta^*\rangle+\xi_i}{\lambda_o \sqrt{n}}\right|\geq 1/2}}\nonumber\\
		&\stackrel{(c)}{\leq}  4\frac{K^4} {\tau_{\vecx}^2}\|\vecv\|_1^2+K^2\|\Sigma^\frac{1}{2}\vecv\|_2^2\sqrt{\frac{2}{\lambda_o\sqrt{n}}}\sqrt{\mbb{E}\left(|\xi_i|+|\langle \vecx_i-\tilde{\vecx}_i,\vecbeta^*\rangle|\right)}\nonumber\\
		&\stackrel{(d)}{\leq}  4\frac{K^4} {\tau_{\vecx}^2}\|\vecv\|_1^2+K^2\|\Sigma^\frac{1}{2}\vecv\|_2^2\sqrt{\frac{2}{\lambda_o\sqrt{n}}}\sqrt{\sigma+\|\vecbeta^*\|_1 \frac{K^4}{\tau_{\vecx}^3}}\nonumber \\
		&\stackrel{(e)}{\leq}  \left(4K^4c_{r_1}^2\frac{s} {\tau_{\vecx}^2}+K^2\sqrt{\frac{2}{\lambda_o\sqrt{n}}}\sqrt{\sigma+\|\vecbeta^*\|_1 \frac{K^4}{\tau_{\vecx}^3}}\right)\|\Sigma^\frac{1}{2}\vecv\|_2^2,
	\end{align}
	where (a) follows from Lemma \ref{a:l:1}, (b) follows from H{\"o}lder's inequality,  (c) follows from the definition of $\vecx_i$, the relationship of expectation of indicator function and probability,  Markov's inequality and the triangular inequality, (d) follows from Lemma \ref{a:l:2}, and (e) follows from $r_1 = c_{r_1}\sqrt{s}r_\Sigma$ and $r_\Sigma = \|\Sigma^\frac{1}{2}\vecv\|_2^2$.

	Consequently, from \eqref{ine:huv-conv-f}, \eqref{ine:fbelow}, \eqref{ine:aplower:tmp} \eqref{ine:v2}, \eqref{ine:q2} and \eqref{ine:q3}, we have
	\begin{align}
	\label{ap:h_bellow}
	&\|\Sigma^\frac{1}{2}\vecv\|_2^2\left( 1-K^2\sqrt{\frac{2}{\lambda_o\sqrt{n}}} \left(\sqrt{\sigma+\|\vecbeta^*\|_1\frac{K^4}{\tau_{\vecx}^3}}+\sqrt{1+c_{r_1}\sqrt{s}\frac{K^4}{\tau_{\vecx}^3}}\right)-12K^4 c_{r_1}^2\frac{s} {\tau_{\vecx}^2}\right)-\lambda_o^2\Delta\nonumber \\
	& \quad \quad \leq \lambda_o^2	\sum_{i=1}^n  \left(-h\left(\frac{\langle \vecx_i-\tilde{\vecx}_i, \vecbeta^*\rangle+\xi_i}{\lambda_o \sqrt{n}}-\frac{\tilde{\vecx}_i^\top \vecv}{\lambda_o\sqrt{n}}\right)  +h\left(\frac{\langle \vecx_i-\tilde{\vecx}_i, \vecbeta^*\rangle+\xi_i}{\lambda_o \sqrt{n}}\right)  \right)\frac{\tilde{\vecx}_i^\top \vecv}{\lambda_o\sqrt{n}}.
	\end{align}

	Next we evaluate the stochastic term $\Delta$ defined in \eqref{ap:delta}. 
	From \eqref{ine:f-1/4} and Theorem 2.1 of \cite{Bou2002Bennett}, with probability at least $1-\delta$, we have
	\begin{align}
	\label{ine:delta}
		\Delta & \leq   \mbb{E} \Delta + \sqrt{\sup_{ \vecv \in \mathfrak{V}^\Sigma_{r_1,r_2,r_\Sigma}} \sum_{i=1}^n\mbb{E}  (f_i(\vecv)-\mbb{E}f_i(\vecv))^2} \sqrt{2\log(1/\delta)} + \frac{1}{3}\log(1/\delta).
	\end{align}
	From  \eqref{ine:f-1/4} and $r_1=c_{r_1}\sqrt{s}r_\Sigma$, we have
	\begin{align}
	\label{ap:ine:cov3}
	\mbb{E}(f_i(\vecv)-\mbb{E}f_i(\vecv))^2 \leq \mbb{E}f_i^2(\vecv)\leq  \mbb{E}\frac{f_i(\vecv)}{4} \leq \mbb{E}\frac{\langle \tilde{\vecx}_i,\vecv\rangle^2 }{4\lambda_o^2n}&\leq \frac{1}{\lambda_o^2n}\left(\frac{K^4} {\tau_{\vecx}^2}\|\vecv\|_1^2+  \frac{\|\Sigma^\frac{1}{2}\vecv\|_2^2}{4}\right)\nonumber\\
	&\leq \frac{1}{\lambda_o^2n}\left(K^4c_{r_1}^2\frac{s} {\tau_{\vecx}^2}+  \frac{1}{4}\right)\|\Sigma^\frac{1}{2}\vecv\|_2^2.
	\end{align}
	Combining this and  \eqref{ine:delta},	from the triangular inequality, with probability at least $1-\delta$, we have
	\begin{align}
	\label{ap:delta_upper}
		\lambda_o^2\Delta &\leq \lambda_o^2\mbb{E} \Delta+\lambda_o\sqrt{n} \left(K^2c_{r_1}\sqrt{ s\frac{\log(1/\delta)} {\tau_{\vecx}^2n}}+\sqrt{ \frac{\log(1/\delta)}{2n}}\right)\|\Sigma^\frac{1}{2}\vecv\|_2 +\lambda_o^2\frac{1}{3}\log(1/\delta)\nonumber\\
		&\stackrel{(a)}{\leq} \lambda_o^2\mbb{E} \Delta+\lambda_o\sqrt{n} \left(K^2c_{r_1}\sqrt{ s\frac{\log(d/\delta)} {\tau_{\vecx}^2n}}+\sqrt{ s\frac{\log(d/\delta)}{2n}}\right)\|\Sigma^\frac{1}{2}\vecv\|_2 +\lambda_o^2n \frac{\log(d/\delta)}{3n}\nonumber\\
		&\stackrel{(b)}{\leq} \lambda_o^2\mbb{E} \Delta+\lambda_o\sqrt{n} c_{r_1}\left(\frac{K^2} {\tau_{\vecx}}+1\right)\sqrt{\frac{s\log(d/\delta)}{n}}\|\Sigma^\frac{1}{2}\vecv\|_2,
	\end{align}
	where (a) follows from $\log(1/\delta)\leq \log(d/\delta)$ and $s\geq 1$ and (b) follows from $\lambda_o\sqrt{n}\sqrt{\frac{s\log(d/\delta)}{n}}\leq r_\Sigma$, $s\geq 1$ and $c_{r_1}\geq 1$.
	For $\mbb{E}\Delta$, from symmetrization inequality (Lemma 11.4 of \cite{BouLugMas2013concentration}), we have 
	\begin{align}
		\mbb{E}\Delta &\leq 2   \,\mbb{E} \sup_{ \vecv \in  \mathfrak{V}^\Sigma_{r_1,r_2,r_\Sigma} } \left|  \sum_{i=1}^n 
		b_i \varphi \left(\frac{\tilde{\vecx}_i^\top \vecv}{\lambda_o\sqrt{n}}\right) \psi \left(\frac{\langle \vecx_i-\tilde{\vecx}_i, \vecbeta^*\rangle+\xi_i}{\lambda_o \sqrt{n}}\right)\right|\nonumber\\
		&\leq 2   \,\mbb{E} \sup_{ \vecv \in  r_1 \mbb{B}^d_1 \cap r_2 \mbb{B}^d_2 \cap r_\Sigma \mbb{B}^d_\Sigma } \left|  \sum_{i=1}^n 
		b_i \varphi \left(\frac{\tilde{\vecx}_i^\top \vecv}{\lambda_o\sqrt{n}}\right) \psi \left(\frac{\langle \vecx_i-\tilde{\vecx}_i, \vecbeta^*\rangle+\xi_i}{\lambda_o \sqrt{n}}\right)\right|,
	\end{align} 
	where  $\{b_i\}_{i=1}^n$ is a sequence of i.i.d. Rademacher random variables which are independent of $\{\tilde{\vecx}_i,\xi_i\}_{i=1}^n$.
	We  denote $\mbb{E}^*$ as a conditional expectation of $\left\{b_i\right\}_{i=1}^n$ given $\left\{\tilde{\vecx}_i,\xi_i\right\}_{i=1}^n$. From Exercise 2.2.2 of \cite{Tal2014Upper}, for any $\vecv_0 \in  r_1 \mbb{B}^d_1 \cap r_2 \mbb{B}^d_2 \cap r_\Sigma \mbb{B}^d_\Sigma$, we have 
	\begin{align}
		\label{ine:rad}
		&\mbb{E}^* \sup_{ \vecv  \in r_1 \mbb{B}^d_1 \cap r_2 \mbb{B}^d_2 \cap r_\Sigma \mbb{B}^d_\Sigma } \left|  \sum_{i=1}^n 
		b_i \varphi \left(\frac{\tilde{\vecx}_i^\top \vecv}{\lambda_o\sqrt{n}}\right) \psi \left(\frac{\langle \vecx_i-\tilde{\vecx}_i, \vecbeta^*\rangle+\xi_i}{\lambda_o \sqrt{n}}\right)\right|\nonumber\\
		&\leq 	\mbb{E}^*\left|  \sum_{i=1}^n 
		b_i \varphi \left(\frac{\tilde{\vecx}_i^\top \vecv_0}{\lambda_o\sqrt{n}}\right) \psi \left(\frac{\langle \vecx_i-\tilde{\vecx}_i, \vecbeta^*\rangle+\xi_i}{\lambda_o \sqrt{n}}\right)\right|+	\mbb{E}^* \sup_{ \vecv \in  r_1 \mbb{B}^d_1 \cap r_2 \mbb{B}^d_2 \cap r_\Sigma \mbb{B}^d_\Sigma} \sum_{i=1}^n 
		b_i \varphi \left(\frac{\tilde{\vecx}_i^\top \vecv}{\lambda_o\sqrt{n}}\right) \psi \left(\frac{\langle \vecx_i-\tilde{\vecx}_i, \vecbeta^*\rangle+\xi_i}{\lambda_o \sqrt{n}}\right).
	\end{align}
	For the first term of \eqref{ine:rad}, we set $\vecv_0 =0$.
	For the second term of \eqref{ine:rad}, from contraction principle (Theorem 11.5 of \cite{BouLugMas2013concentration}),  we  have
	\begin{align}
		&\mbb{E} ^*\sup_{\vecv \in r_1 \mbb{B}^d_1 \cap r_2 \mbb{B}^d_2 \cap r_\Sigma \mbb{B}^d_\Sigma} \sum_{i=1}^n b_i \varphi \left(\frac{\tilde{\vecx}_i^\top \vecv}{\lambda_o\sqrt{n}}\right) \psi \left(\frac{\langle \vecx_i-\tilde{\vecx}_i, \vecbeta^*\rangle+\xi_i}{\lambda_o \sqrt{n}}\right)     \leq	\mbb{E}^* \sup_{\vecv\in r_1 \mbb{B}^d_1 \cap r_2 \mbb{B}^d_2 \cap r_\Sigma \mbb{B}^d_\Sigma}  \sum_{i=1}^n   b_i \varphi \left(\frac{\tilde{\vecx}_i^\top \vecv}{\lambda_o\sqrt{n}}\right),
	\end{align}	
	and from the fact that $\varphi$ is $\frac{1}{2}$-Lipschitz and $\varphi(0)=0$, and contraction principle (Theorem 11.6 in \cite{BouLugMas2013concentration}), 
	\begin{align}
		\mbb{E}^* \sup_{\vecv \in  r_1 \mbb{B}^d_1 \cap r_2 \mbb{B}^d_2 \cap r_\Sigma \mbb{B}^d_\Sigma} \sum_{i=1}^n b_i \varphi \left(\frac{\tilde{\vecx}_i^\top \vecv}{\lambda_o\sqrt{n}}\right) &\leq\frac{1}{\lambda_o\sqrt{n}}	\mbb{E}^*\sup_{\vecv \in  r_1 \mbb{B}^d_1 \cap r_2 \mbb{B}^d_2 \cap r_\Sigma \mbb{B}^d_\Sigma}   \sum_{i=1}^n   b_i \tilde{\vecx}_i^\top \vecv .
	\end{align}
	and from the basic property of the expectation, we have
	\begin{align}
		\label{ine:hub-stoc-upper}
		&\lambda_o^2  \mbb{E} \sup_{ \vecv \in r_1 \mbb{B}^d_1 \cap r_2 \mbb{B}^d_2 \cap r_\Sigma \mbb{B}^d_\Sigma} \left| \sum_{i=1}^n b_i \varphi \left(\frac{\tilde{\vecx}_i^\top \vecv}{\lambda_o\sqrt{n}}\right) \psi \left(\frac{\langle \vecx_i-\tilde{\vecx}_i, \vecbeta^*\rangle+\xi_i}{\lambda_o \sqrt{n}}\right)    \right|\nonumber\\
		&\leq	
		\lambda_o\sqrt{n}	\mbb{E}\sup_{\vecv\in r_1 \mbb{B}^d_1 \cap r_2 \mbb{B}^d_2 \cap r_\Sigma \mbb{B}^d_\Sigma}  \sum_{i=1}^n  \frac{1}{n} b_i \tilde{\vecx}_i^\top \vecv\nonumber\\
		&\stackrel{(a)}{\leq}	
		\lambda_o \sqrt{n}\left(\sqrt{2\frac{\log d}{n}}+\tau_\vecx\frac{\log d}{n}\right)r_1\nonumber\\
		&\stackrel{(b)}{\leq}	 \lambda_o\sqrt{n}c_{r_1} \left(\sqrt{2}
		+\tau_\vecx\sqrt{\frac{\log (d/\delta)}{n}}\right)\sqrt{\frac{s\log(d/\delta)}{n}}\|\Sigma^\frac{1}{2}\vecv\|_2,
	\end{align}	
	where (a) follows from Lemma \ref{l:1e} and (b) follows from $\log(d)\leq \log(d/\delta)$, $1\leq \log(d/\delta)$, $r_1 = c_{r_1}\sqrt{s}r_\Sigma$ and $c_{r_1}\geq 1$.

	Combining \eqref{ap:delta_upper} and  \eqref{ine:hub-stoc-upper}, we have
	\begin{align}
		\label{ap:delta_upper2}
			\lambda_o^2\Delta &\leq 3\lambda_o\sqrt{n} c_{r_1} \left(1+\frac{K^2}{\tau_\vecx}
			+\tau_\vecx\sqrt{\frac{\log (d/\delta)}{n}}\right) \sqrt{\frac{s\log(d/\delta)}{n}}\|\Sigma^\frac{1}{2}\vecv\|_2,
		\end{align}
		and combining  \eqref{ap:h_bellow},  the proof is complete.
\end{proof}

\section{Proofs of  Propositions \ref{p:cwpre}, \ref{p:main:out} and   \ref{p:main:out2}, and Lemma \ref{l:w2}}
\label{sec:proof:sec:keypropositions}
\subsection{Proof of Proposition \ref{p:cwpre}}
\label{sec:robust-pre}
\begin{proof}
	We note that this proof is similar to the one of Lemma 2 of \cite{FanWanZhu2021Shrinkage}.
	For any $M\in \mathfrak{M}_{r}$, we have
	\begin{align}
		\frac{1}{n} \sum_{i=1}^n  \langle\tilde{\vecx}_i \tilde{\vecx}_i^\top,M\rangle = \underbrace{\frac{1}{n} \sum_{i=1}^n  \langle\tilde{\vecx}_i \tilde{\vecx}_i^\top,M\rangle -\mbb{E}\frac{1}{n} \sum_{i=1}^n   \langle\tilde{\vecx}_i \tilde{\vecx}_i^\top,M\rangle}_{T_4}+ \mbb{E}\frac{1}{n} \sum_{i=1}^n  \langle\tilde{\vecx}_i \tilde{\vecx}_i^\top,M\rangle.
	\end{align}
	First, we evaluate $T_4$.
	We note that, for any $1\leq j_1,j_2\leq d$,
	\begin{align}
		\mbb{E} \tilde{x}_{i_{j_1}}^2\tilde{x}_{i_{j_2}}^2 &\leq \sqrt{\mbb{E}\tilde{x}_{i_{j_1}}^4} \sqrt{\mbb{E}\tilde{x}_{i_{j_2}}^4} \leq K^4,\quad \mbb{E} \tilde{x}_{i_{j_1}}^{2p}\tilde{x}_{i_{j_2}}^{2p}\leq 
		\tau_\vecx^{2(p-2)}\mbb{E} \tilde{x}_{i_{j_1}}^2\tilde{x}_{i_{j_2}}^2 \leq \tau_\vecx^{2(p-2)} K^4.
	\end{align}
	From Bernstein's inequality (Lemma 5.1 of \cite{Dir2015Tail}), we have
	\begin{align}
		 \mbb{P}\left(\frac{1}{n}\sum_{i=1}^n\tilde{\vecx}_{i_j}\tilde{\vecx}_{i_j}^\top -\mbb{E}\sum_{i=1}^n\frac{1}{n}\tilde{\vecx}_{i_j}\tilde{\vecx}_{i_j}^\top \geq K^2\sqrt{2\frac{t}{n}}+\frac{\tau_{\vecx}^2t}{n}\right)\leq e^{-t}.
	\end{align}
	From the union bound, we have
	\begin{align}
		\mbb{P}\left(\left\|\frac{1}{n}\sum_{i=1}^n\tilde{\vecx}_{i_j}\tilde{\vecx}_{i_j}^\top -\frac{1}{n}\sum_{i=1}^n\mbb{E}\tilde{\vecx}_{i_j}\tilde{\vecx}_{i_j}^\top \right\|_\infty \leq \sqrt{2}K^2\sqrt{\frac{\log(d/\delta)}{n}}+\tau_{\vecx}^2 \frac{\log(d/\delta)}{n}\right)&\geq 1-\delta.
	\end{align}
	From H{\"o}lder's inequality, we have
	\begin{align}
		\mbb{P}\left(T_1\leq \sqrt{2}K^2\sqrt{\frac{\log(d/\delta)}{n}}\|M\|_1+\tau_{\vecx}^2 \frac{\log(d/\delta)}{n}\|M\|_1\right)&\geq 1-\delta.
	\end{align}
	Next, we evaluate $\mbb{E} \left\langle\tilde{\vecx}_i \tilde{\vecx}_i^\top,M\right\rangle$. 
	We have
	\begin{align}
		\label{ine:tilde}
		\mbb{E} \left\langle\tilde{\vecx}_i \tilde{\vecx}_i^\top,M\right\rangle
		&=\mbb{E} \left\langle\tilde{\vecx}_i \tilde{\vecx}_i^\top-\vecx_i \vecx_i^\top,M\right\rangle+\mbb{E} \left\langle\vecx_i \vecx_i^\top-\Sigma,M\right\rangle+\mbb{E} \left\langle \Sigma,M\right\rangle\nonumber\\
		&=\mbb{E} \left\langle\tilde{\vecx}_i \tilde{\vecx}_i^\top-\vecx_i \vecx_i^\top,M\right\rangle+\mbb{E} \left\langle \Sigma,M\right\rangle
	\end{align}
	From  H{\"o}lder's inequality and the positive semi-definiteness of $M$, we have
	\begin{align}
		\mbb{E} \left\langle \Sigma,M\right\rangle\leq \|\Sigma\|_{\mr{op}}\|M\|_*= \|\Sigma\|_{\mr{op}}\mr{Tr}(M) = \|\Sigma\|_{\mr{op}}r^2.
	\end{align}
	From \eqref{eq1:a:l:1} and H{\"o}lder's inequality, we have
	\begin{align}
		\mbb{E} \left\langle\tilde{\vecx}_i \tilde{\vecx}_i^\top-\vecx_i \vecx_i^\top,M\right\rangle\leq 2\frac{K^4}{\tau_{\vecx}^2}\|M\|_1.
	\end{align}
	Finally, combining the arguments above, with probability at least $1-\delta$, we have
	\begin{align}
		\label{ine:peeling3}
		\left|\sum_{i=1}^n \frac{\left\langle\tilde{\vecx}_i \tilde{\vecx}_i^\top,M\right\rangle}{n}\right|
		&\leq  \left(\sqrt{2}K^2\sqrt{\frac{\log(d/\delta)}{n}}+\tau_{\vecx}^2 \frac{\log(d/\delta)}{n}+2\frac{K^4}{\tau_{\vecx}^2}\right)\|M\|_1+\|\Sigma\|_{\mr{op}}r^2.
	\end{align}
\end{proof}

Define $\mathfrak{M}_{\vecv,r_2} = \{M\in \mbb{R}^{d\times d}\,:\, M = \vecv \vecv^\top,\,\|\vecv\|_2 = r_2\}$.
\subsection{Proof of Proposition \ref{p:main:out}}
	We note that, from H{\"o}lder's inequality and $|w_i|\leq 1/n$, 
	\begin{align}
		\left|\sum_{i \in \mc{O}}\hat{w}_iu_i \tilde{\vecX}_i^\top\vecv  \right|^2 &\stackrel{(a)}{\leq}   c^2\frac{o}{n}\sum_{i =1}^n\hat{w}_i |\tilde{\vecX}_i^\top \vecv|^2.
	\end{align}
	We focus on $\sum_{i \in \mc{O}}\hat{w}_i |\tilde{\vecX}_i^\top \vecv|^2$.
	For any $\vecv \in \mbb{R}^d$ such that $\|\vecv\|_2= r_2$, 
	\begin{align}
		\sum_{i =1}^n\hat{w}_i (\tilde{\vecX}_i^\top \vecv)^2 &= \sum_{i =1}^n\hat{w}_i (\tilde{\vecX}_i^\top \vecv)^2 -\lambda_*\|\vecv\|_1^2+\lambda_*\|\vecv\|_1^2\nonumber\\
		&\stackrel{(a)}{\leq} \sup_{M\in \mathfrak{M}_{r}}\left(\sum_{i =1}^n\hat{w}_i\left\langle \tilde{\vecX}_i\tilde{\vecX}_i^\top, M\right\rangle -\lambda_*\|M\|_1\right)+\lambda_*\|\vecv\|_1^2\nonumber\\
		&\stackrel{(b)}{\leq}  \tau_{suc}+\lambda_*\|\vecv\|_1^2,
	\end{align}
	where (a) follows from the fact that $\mathfrak{M}_{\vecv,r_2} \subset \mathfrak{M}_{r_2}$, and (b) follows from \eqref{ine:optM}.
	Combining the arguments above, for any $\vecv \in   r_1 \mbb{B}^d_1 \cap r_2 \mbb{B}^d_2 \cap r_\Sigma \mbb{B}^d_\Sigma$, we have
	\begin{align}
		\sum_{i \in \mc{O}}\hat{w}_iu_i \tilde{\vecX}_i^\top\vecv  &\stackrel{(a)}{\leq} \sqrt{2}  c\sqrt{\frac{o}{n}} \sqrt{\tau_{suc}} + \sqrt{2}c \sqrt{\frac{o}{n}} \sqrt{\lambda_*} \|\vecv\|_1\nonumber\\
		&= \sqrt{2}  c \sqrt{\frac{o}{n}}  \sqrt{\frac{\|\Sigma\|_{\mr{op}}}{1-\varepsilon}r^2_2} + \sqrt{2}c \sqrt{\frac{o}{n}} \sqrt{c_*\left(\sqrt{2}K^2\sqrt{\frac{\log(d/\delta)}{n}}+\tau_{\vecx}^2 \frac{\log(d/\delta)}{n}+2\frac{K^4}{\tau_{\vecx}^2}\right)} \|\vecv\|_1\nonumber\\
		&\stackrel{(b)}{\leq} 2cc_*'\left( c_{r_2}\|\Sigma^\frac{1}{2}\|_{\mr{op}}\sqrt{\frac{o}{n}}  +  c_{r_1}\sqrt{s}\sqrt{K^2\sqrt{\frac{\log(d/\delta)}{n}}+\tau_{\vecx}^2 \frac{\log(d/\delta)}{n}+\frac{K^4}{\tau_{\vecx}^2}} \sqrt{\frac{o}{n}}\right)r_\Sigma,
	\end{align}
	where (a) follows the triangular inequality and (b) follows from $c_{*}' = \max\{\frac{1}{1-\varepsilon},c_*\}$, $r_1 = c_{r_1}\sqrt{s}r_\Sigma$ and $r_2 = c_{r_2}r_\Sigma$.

\subsection{Proof of Proposition \ref{p:main:out2}}
	We note that, from H{\"o}lder's inequality, we have 
	\begin{align}
		\left|\sum_{i \in I_m}\frac{u_i \tilde{\vecx}_i^\top\vecv }{n} \right|^2  &\leq \sum_{i \in I_m}\frac{1}{n}u_i^2 \sum_{i \in I_m}\frac{1}{n}(\tilde{\vecx}_i^\top\vecv)^2  \leq c^2\frac{m}{n}\sum_{i =1}^n\frac{1}{n} (\tilde{\vecx}_i^\top\vecv)^2 .
	\end{align}
	From the proof of Proposition \ref{p:cwpre} and  H{\"o}lder's inequality, for any $ \vecv\in  r_1 \mbb{B}^d_1 \cap r_2 \mbb{B}^d_2 \cap r_\Sigma \mbb{B}^d_\Sigma$, we have
	\begin{align}
		\sum_{i =1}^n\frac{(\tilde{\vecx}_i^\top\vecv)^2 }{n} \leq   \left(2K^2\sqrt{\frac{\log(d/\delta)}{n}}+\tau_{\vecx}^2 \frac{\log(d/\delta)}{n}+2\frac{K^4}{\tau_{\vecx}^2}\right)r_1^2+\|\Sigma\|_{\mr{op}}r_2^2.
	\end{align}
	From triangular inequality, $r_1 = c_{r_1}\sqrt{s}r_\Sigma$ and $r_2 = c_{r_2}r_\Sigma$,  for any $ \vecv\in  r_1 \mbb{B}^d_1 \cap r_2 \mbb{B}^d_2 \cap r_\Sigma \mbb{B}^d_\Sigma$, we have
	\begin{align}
		\sum_{i \in I_m}\frac{1}{n}u_i \tilde{\vecx}_i^\top\vecv  \leq 2c\left(c_{r_2}\|\Sigma^\frac{1}{2}\|_{\mr{op}}\sqrt{\frac{m}{n}} + c_{r_1}\sqrt{s}\sqrt{K^2\sqrt{\frac{\log(d/\delta)}{n}}+\tau_{\vecx}^2 \frac{\log(d/\delta)}{n}+\frac{K^4}{\tau_{\vecx}^2}}\sqrt{\frac{m}{n}}\right)r_\Sigma.
	\end{align}

\section{Proofs of Theorems \ref{t:main:no}  and \ref{t:main}}
 \subsection{Proof of Theorem \ref{t:main:no}}
To prove Theorem \ref{t:main:no},
it is sufficient that we confirm \eqref{ine:det:main:0-1} - \eqref{ine:det:main:0-4} hold under the assumption of Theorem \ref{t:main:no}.
\subsubsection{Confirming \eqref{ine:det:main:0-1} and \eqref{ine:det:main:0-2}}
First, we confirm \eqref{ine:det:main:0-1}.
We note that, from $\tau_\vecx = \sqrt{\frac{n}{\log(d/\delta)}}$,
\begin{align}
	\label{ine:condno}
	 K^4c_{r_1}\frac{\sqrt{s}}{\tau_{\vecx}^3}\stackrel{(a)}{\leq} c_{r_1}\sqrt{s\frac{\log(d/\delta)}{n}},\quad
	 K^4 s^\frac{3}{4}\left(\frac{1}{\tau_{\vecx}}\right)^{ 3-\frac{1}{4}}\stackrel{(b)}{\leq}  c_{r_1}\sqrt{s\frac{\log(d/\delta)}{n}},
\end{align}
where (a)  and (b) follows from \eqref{ine:tp2:no}.
Then, from Proposition \ref{p:main1:no}, for any $\vecv \in r_1 \mbb{B}^d_1 \cap r_2 \mbb{B}^d_2 \cap r_\Sigma \mbb{B}^d_\Sigma$, we have
\begin{align}
	\label{ine:t:main:no:1}
		&\sup_{\vecv  \in r_1 \mbb{B}^d_1 \cap r_2 \mbb{B}^d_2 \cap r_\Sigma \mbb{B}^d_\Sigma}\left|\sum_{i=1}^n  \frac{1}{n}h\left(\frac{y_i-\langle \tilde{\vecx}_i, \vecbeta^*\rangle}{\lambda_o \sqrt{n}}\right)\langle \tilde{\vecx}_i,\vecv\rangle \right|\leq 16c_{r_1} \sqrt{s\frac{\log(d/\delta)}{n}}r_\Sigma.
	\end{align}

Next, we confirm \eqref{ine:det:main:0-2}.
We note that, from $\tau_\vecx = \sqrt{\frac{n}{\log(d/\delta)}}$,
\begin{align}
	&K^2\sqrt{\frac{2}{\lambda_o\sqrt{n}}} \left(\sqrt{\sigma+\|\vecbeta^*\|_1\frac{K^4}{\tau_{\vecx}^3}}+\sqrt{1+c_{r_1}\sqrt{s}\frac{K^4}{\tau_{\vecx}^3}}\right)+12K^4 c_{r_1}^2\frac{s} {\tau_{\vecx}^2} \stackrel{(a)}{\leq } 	K^2\sqrt{\frac{18(\sigma+1)}{\lambda_o\sqrt{n}}}+\frac{1}{6} \stackrel{(b)}{\leq }  \frac{1}{2},\nonumber\\
	&2+\frac{K^2}{\tau_\vecx}+\tau_\vecx\sqrt{\frac{\log (d/\delta)}{n}}\stackrel{(c)}{\leq } 4,
\end{align}
where (a) and (c) follows from \eqref{ine:tp2:no}, (b) follows from the definition of $\lambda_o$.

Then, from Proposition \ref{p:main:sc:no}, for any $\vecv \in r_1 \mbb{B}^d_1 \cap r_2 \mbb{B}^d_2$ such that $\|\Sigma^\frac{1}{2}\vecv \|_2 = r_\Sigma$, we have
\begin{align}
	\label{ine:p:main:sc:no:1}
	& \sum_{i=1}^n \frac{\lambda_o}{\sqrt{n}}\left(-h\left(\frac{y_i-\langle \tilde{\vecx}_i,\vecbeta^*+\vecv\rangle}{\lambda_o \sqrt{n}}\right) +h\left(\frac{y_i-\langle \tilde{\vecx}_i, \vecbeta^*\rangle}{\lambda_o \sqrt{n}}\right)\right) \tilde{\vecx}_i^\top \vecv\geq  \frac{\|\Sigma^\frac{1}{2}\vecv\|_2^2}{2}-12\lambda_o \sqrt{n} c_{r_1}\sqrt{\frac{s\log (d/\delta)}{n}}\|\Sigma^\frac{1}{2}\vecv\|_2.
\end{align}

Therefore, we see that \eqref{ine:det:main:0-1} and  \eqref{ine:det:main:0-2} hold with
\begin{align}
	\label{eq:no:b}
	r_{a,\Sigma} = r_{b,\Sigma} = 16\lambda_o \sqrt{n} c_{r_1}\sqrt{\frac{s\log (d/\delta)}{n}},\quad b = 1/2.
\end{align}

\subsubsection{Confirming \eqref{ine:det:main:0-3}}
Thirdly, we confirm \eqref{ine:det:main:0-3}. From \eqref{eq:no:b},
we see that
\begin{align}
		\frac{r_{a,\Sigma}}{c_{r_1}\sqrt{s}} 
		&=16\lambda_o\sqrt{n}\sqrt{\frac{\log(d/\delta)}{n}} \quad\text{and}\quad 
			\frac{\lambda_s +  \frac{ r_{a,\Sigma}}{c_{r_1}\sqrt{s}}}{\lambda_s -  \frac{r_{a,\Sigma}}{c_{r_1}\sqrt{s}} } = \frac{c_s\frac{c_{\mr{RE}}+1}{c_{\mr{RE}}-1}+1}{c_s\frac{c_{\mr{RE}}+1}{c_{\mr{RE}}-1}-1}\leq c_{\mr{RE}}.
		\end{align}

	\subsubsection{Confirming \eqref{ine:det:main:0-4}}
		Lastly, we confirm \eqref{ine:det:main:0-4}. From  \eqref{eq:no:b}, the definition of $\lambda_s$  and $c_{\mr{RE}}>1$, we have
		\begin{align}
			&\frac{2}{b} \left(r_{a,\Sigma}+r_{b,\Sigma} + c_{r_1}\sqrt{s}\lambda_s\right)\leq 12c_s \frac{c_{\mr{RE}}+1}{c_{\mr{RE}}-1} r_{a,\Sigma} =12 c_{r_1}\sqrt{s}\lambda_s,
			\end{align}
			and we see the condition about $r_\Sigma$ is satisfied. From the definition, conditions about $r_1$ and $r_2$ are clearly satisfied.

			\subsection{Proof of Theorem \ref{t:main}}
			\label{sec:proof}
			To prove Theorem \ref{t:main},
			it is sufficient that we confirm \eqref{ine:det:main:0-1} - \eqref{ine:det:main:0-4} hold under the assumption of Theorem \ref{t:main}.

			\subsubsection{Confirming \eqref{ine:det:main:0-1} and \eqref{ine:det:main:0-2} }
			We note that, from $r_\Sigma \leq 1$, $\lambda_o\sqrt{n}\geq 1$, we have $(\log(d/\delta))/n\leq 1$.
			From  \eqref{ine:tp2}, we have
			\begin{align}
				\label{assumption}
				\max \left\{K^4 \left(\frac{\log(d/\delta)}{n}\right)^\frac{1}{4},  \frac{K^4}{c_{r_1}}s^\frac{1}{4}\left(\frac{\log(d/\delta)}{n}\right)^{ \frac{3}{16}}\right\}\leq 1.
			\end{align}
			From $K\geq 1$ and  \eqref{ine:tp2}, we have
			\begin{align}
				\label{assumption1}
				(K^4+K^2+1)\frac{c_{r_1}^2}{c_{r_2}^2\|\Sigma^\frac{1}{2}\|_{\mr{op}}} s \sqrt{\frac{\log(d/\delta)}{n}}\leq 1.
			\end{align}
			From  $(\log(d/\delta))/n\leq 1$ and \eqref{ine:tp2-2}, we have
			\begin{align}
				\label{assumption2}
				\max \left\{ K^2 \left(\frac{\log(d/\delta)}{n}\right)^\frac{1}{4}, K^4 \max\{ \|\vecbeta^*\|_1, c_{r_1}\sqrt{s}\} \left(\frac{\log(d/\delta)}{n}\right)^\frac{3}{4},72 K^4c_{r_1}^2 s \sqrt{\frac{\log(d/\delta)}{n}}\right\}\leq 1.
			\end{align}

			First, we confirm \eqref{ine:det:main:0-1}.
			From  $\{1,\cdots,n\} = \mc{I}\cap I_\geq+\mc{I}\cap I_< + \mc{O}$, we see
			\begin{align}
				\label{ine:outin}
				& \sum_{i=1}^n \frac{1}{n}h\left(\frac{\hat{w}_iW_i-\langle \hat{w}_i\tilde{\vecX}_i, \vecbeta^*\rangle}{\lambda_o \sqrt{n}}\right) \hat{w}_i\tilde{\vecX}_i^\top \vectheta_\eta \nonumber\\
				&\leq  \left|\sum_{i=1}^n \frac{1}{n}h\left(\frac{y_i-\langle \tilde{\vecx}_i, \vecbeta^*\rangle}{\lambda_o \sqrt{n}}\right) \tilde{\vecx}_i^\top \vectheta_\eta\right|+ \left|\sum_{i\in \mc{O}\cup (\mc{I}\cap I_<)} \frac{1}{n}h\left(\frac{y_i-\langle \tilde{\vecx}_i, \vecbeta^*\rangle}{\lambda_o \sqrt{n}}\right) \tilde{\vecx}_i^\top \vectheta_\eta \right|+  \left|\sum_{i \in \mc{O}} h\left(\frac{\hat{w}_iW_i-\langle \hat{w}_i\tilde{\vecX}_i, \vecbeta^*\rangle}{\lambda_o \sqrt{n}}\right) \hat{w}_i\tilde{\vecX}_i^\top \vectheta_\eta\right|.
			\end{align}

We note that, from the definition of $\tau_\vecx$ and $(\log(d/\delta))/n\leq 1$ and \eqref{assumption}, we have
\begin{align}
	\tau_\vecx \frac{\sqrt{s}\log (d/\delta)}{n} \leq \sqrt{s\frac{\log (d/\delta)}{n}},\quad
	 K^4 \frac{\sqrt{s}}{\tau_{\vecx}^3} \leq \sqrt{s\frac{\log(d/\delta)}{n}},\quad K^4 s^\frac{3}{4}\left(\frac{1}{\tau_{\vecx}}\right)^{ 3-\frac{1}{4}}\leq  c_{r_1}\sqrt{s\frac{\log(d/\delta)}{n}}.
\end{align}
Then, from Proposition \ref{p:main1:no}, for any $\vecv \in r_1 \mbb{B}^d_1 \cap r_2 \mbb{B}^d_2 \cap r_\Sigma \mbb{B}^d_\Sigma$, we have
\begin{align}
	\label{ine:t:main:1}
		&\sup_{\vecv  \in r_1 \mbb{B}^d_1 \cap r_2 \mbb{B}^d_2 \cap r_\Sigma \mbb{B}^d_\Sigma}\left|\sum_{i=1}^n  \frac{1}{n}h\left(\frac{y_i-\langle \tilde{\vecx}_i, \vecbeta^*\rangle}{\lambda_o \sqrt{n}}\right)\langle \tilde{\vecx}_i,\vecv\rangle \right|\leq 16c_{r_1} \sqrt{s\frac{\log(d/\delta)}{n}}r_\Sigma.
	\end{align}

			We note that, from the definition of $\tau_\vecx$ and from \eqref{assumption1}, we have
			\begin{align}
				\label{ine:robust}
				c_{r_1}\sqrt{s}\sqrt{K^2\sqrt{\frac{\log(d/\delta)}{n}}+\tau_{\vecx}^2 \frac{\log(d/\delta)}{n}+\frac{K^4}{\tau_{\vecx}^2}}\leq		c_{r_2}\|\Sigma^\frac{1}{2}\|_{\mr{op}}.
			\end{align}
			Additionally, note that $-1\leq h(\cdot)\leq 1$. Then from Proposition \ref{p:main:out} with $c=1$, and 
			from Proposition \ref{p:main:out2} with $m = 3o$ and $c=1$, we have
			\begin{align}
				\max\left\{ \left|\sum_{i \in \mc{O}} h\left(\frac{\hat{w}_iW_i-\langle \hat{w}_i\tilde{\vecX}_i, \vecbeta^*\rangle}{\lambda_o \sqrt{n}}\right) \hat{w}_i\tilde{\vecX}_i^\top \vectheta_\eta\right|,\left|\sum_{i\in \mc{O}\cup (\mc{I}\cap I_<)} \frac{1}{n}h\left(\frac{y_i-\langle \tilde{\vecx}_i, \vecbeta^*\rangle}{\lambda_o \sqrt{n}}\right) \tilde{\vecx}_i^\top \vectheta_\eta \right|\right\}&\leq  2c_*' c_{r_2}\|\Sigma^\frac{1}{2}\|_{\mr{op}} \sqrt{\frac{o}{n}} r_\Sigma.
			\end{align}
			From the arguments above and from $c_*\geq 1$, for any  $\vecv \in r_1 \mbb{B}^d_1 \cap r_2 \mbb{B}^d_2 \cap r_\Sigma \mbb{B}^d_\Sigma$ we have
			\begin{align}
				& \sum_{i=1}^n \frac{1}{n}h\left(\frac{\hat{w}_iW_i-\langle \hat{w}_i\tilde{\vecX}_i, \vecbeta^*\rangle}{\lambda_o \sqrt{n}}\right) \hat{w}_i\tilde{\vecX}_i^\top \vecv \leq 16c_{r_1}\sqrt{\frac{s\log(d/\delta)}{n}}r_\Sigma+ 4c_*' c_{r_2}\|\Sigma^\frac{1}{2}\|_{\mr{op}}\sqrt{\frac{o}{n}} r_\Sigma.
			\end{align}

	Next, we confirm \eqref{ine:det:main:0-2}.
			From the same calculation of \eqref{ine:outin}, we have
			\begin{align}
				&\lambda_o\sqrt{n} \sum_{i=1}^n \left(-h\left(\frac{\hat{w}_iW_i-\langle \hat{w}_i\tilde{\vecX}_i,\vecbeta^*+\vectheta_\eta\rangle}{\lambda_o \sqrt{n}}\right) +h\left(\frac{\hat{w}_iW_i-\langle \hat{w}_i\tilde{\vecX}_i, \vecbeta^*\rangle}{\lambda_o \sqrt{n}}\right)\right)\hat{w}_i\tilde{\vecX}_i^\top \vectheta_\eta\nonumber \\
				&\geq  \lambda_o\sqrt{n} \sum_{i=1}^n \frac{1}{n}\left(-h\left(\frac{y_i-\langle \tilde{\vecx}_i,\vecbeta^*+\vectheta_\eta\rangle}{\lambda_o \sqrt{n}}\right) +h\left(\frac{y_i-\langle \tilde{\vecx}_i, \vecbeta^*\rangle}{\lambda_o \sqrt{n}}\right)\right)\tilde{\vecx}_i^\top \vectheta_\eta\nonumber\\
				&-\left|\lambda_o\sqrt{n} \sum_{i \in \mc{O} \cup (\mc{I} \cap I_<)} \frac{1}{n}\left(-h\left(\frac{y_i-\langle \tilde{\vecx}_i,\vecbeta^*+\vectheta_\eta\rangle}{\lambda_o \sqrt{n}}\right) +h\left(\frac{y_i-\langle \tilde{\vecx}_i, \vecbeta^*\rangle}{\lambda_o \sqrt{n}}\right)\right)\tilde{\vecx}_i^\top \vectheta_\eta\right|\nonumber\\
				&-\left|\lambda_o\sqrt{n} \sum_{i \in \mc{O}} \left(-h\left(\frac{\hat{w}_iW_i-\langle \hat{w}_i\tilde{\vecX}_i,\vecbeta^*+\vectheta_\eta\rangle}{\lambda_o \sqrt{n}}\right) +h\left(\frac{\hat{w}_iW_i-\langle \hat{w}_i\tilde{\vecX}_i, \vecbeta^*\rangle}{\lambda_o \sqrt{n}}\right)\right)\hat{w}_i\tilde{\vecX}_i^\top \vectheta_\eta\right|.
			\end{align}

We note that
\begin{align}
	&K^2\sqrt{\frac{2}{\lambda_o\sqrt{n}}} \left(\sqrt{\sigma+\|\vecbeta^*\|_1\frac{K^4}{\tau_{\vecx}^3}}+\sqrt{1+c_{r_1}\sqrt{s}\frac{K^4}{\tau_{\vecx}^3}}\right)+12K^4 c_{r_1}^2\frac{s} {\tau_{\vecx}^2} \stackrel{(a)}{\leq } 	K^2\sqrt{\frac{18(\sigma+1)}{\lambda_o\sqrt{n}}}+\frac{1}{6} \stackrel{(b)}{\leq }  \frac{1}{2},\nonumber\\
	&2+\frac{K^2}{\tau_\vecx}+\tau_\vecx\sqrt{\frac{\log (d/\delta)}{n}}\stackrel{(c)}{\leq }4,
\end{align}
where (a) and (c) follow from the definition of $\tau_\vecx$ and from \eqref{assumption2}, and  (b) follows from the definition of $\lambda_o$. 
Then, from Proposition \ref{p:main:sc:no}, for any $\vecv \in r_1 \mbb{B}^d_1 \cap r_2 \mbb{B}^d_2$ such that $\|\Sigma^\frac{1}{2}\vecv \|_2 = r_\Sigma$, we have
\begin{align}
	& \sum_{i=1}^n \frac{\lambda_o}{\sqrt{n}}\left(-h\left(\frac{y_i-\langle \tilde{\vecx}_i,\vecbeta^*+\vecv\rangle}{\lambda_o \sqrt{n}}\right) +h\left(\frac{y_i-\langle \tilde{\vecx}_i, \vecbeta^*\rangle}{\lambda_o \sqrt{n}}\right)\right) \tilde{\vecx}_i^\top \vecv\geq  \frac{\|\Sigma^\frac{1}{2}\vecv\|_2^2}{2}-15\lambda_o \sqrt{n} c_{r_1}\sqrt{\frac{s\log (d/\delta)}{n}}\|\Sigma^\frac{1}{2}\vecv\|_2.
\end{align}
From \eqref{ine:robust}, Propositions \ref{p:main:out} and \ref{p:main:out2} with $m=3o$ and $c=2$ and $c_*'\geq 1$, for any $\vectheta_\eta \in r_1 \mbb{B}^d_1 \cap r_2 \mbb{B}^d_2$ such that $\|\Sigma^\frac{1}{2}\vectheta_\eta \|_2 = r_\Sigma$, we  have
			\begin{align}
				&\left|\lambda_o\sqrt{n} \sum_{i \in \mc{O} \cup (\mc{I} \cap I_<)} \frac{1}{n}\left(-h\left(\frac{y_i-\langle \tilde{\vecx}_i,\vecbeta^*+\vectheta_\eta\rangle}{\lambda_o \sqrt{n}}\right) +h\left(\frac{y_i-\langle \tilde{\vecx}_i, \vecbeta^*\rangle}{\lambda_o \sqrt{n}}\right)\right)\tilde{\vecx}_i^\top \vectheta_\eta\right|\nonumber\\
				&+\left|\lambda_o\sqrt{n} \sum_{i \in \mc{O}} \left(-h\left(\frac{\hat{w}_iW_i-\langle \hat{w}_i\tilde{\vecX}_i,\vecbeta^*+\vectheta_\eta\rangle}{\lambda_o \sqrt{n}}\right) +h\left(\frac{\hat{w}_iW_i-\langle \hat{w}_i\tilde{\vecX}_i, \vecbeta^*\rangle}{\lambda_o \sqrt{n}}\right)\right)\hat{w}_i\tilde{\vecX}_i^\top \vectheta_\eta\right|\leq  16c_*' c_{r_2}\|\Sigma^\frac{1}{2}\|_{\mr{op}} \sqrt{\frac{o}{n}} r_\Sigma.
			\end{align}

			Therefore, we see that \eqref{ine:det:main:0-1} and \eqref{ine:det:main:0-2} holds with $b= 1/2$ and 
			\begin{align}
				\label{eq:b}
				r_{a,\Sigma}= r_{b,\Sigma} =\lambda_o \sqrt{n}\left(16c_{r_1}\sqrt{\frac{s\log (d/\delta)}{n}}+16c_*' c_{r_2}\|\Sigma^\frac{1}{2}\|_{\mr{op}}\sqrt{\frac{o}{n}}\right).
			\end{align}
			
	\subsubsection{Confirming \eqref{ine:det:main:0-3}}
			Thirdly, we confirm \eqref{ine:det:main:0-3}. From \eqref{eq:b}, we see that
			\begin{align}
					\frac{ r_{a,\Sigma}}{c_{r_1}\sqrt{s}} = \lambda_o \sqrt{n}\left(16\sqrt{\frac{\log (d/\delta)}{n}}+16c_*' \frac{c_{r_2}}{c_{r_1}}\|\Sigma^\frac{1}{2}\|_{\mr{op}}\sqrt{\frac{o}{sn}}\right)\quad \text{and}\quad
						\frac{\lambda_s +  \frac{r_{a,\Sigma}}{c_{r_1}\sqrt{s}}}{\lambda_s -  \frac{r_{a,\Sigma}}{c_{r_1}\sqrt{s}} } \leq c_{\mr{RE}}.
					\end{align}
							
	\subsubsection{Confirming \eqref{ine:det:main:0-4}}
	From 		Lastly, we confirm \eqref{ine:det:main:0-4}. From \eqref{eq:b}, the definition of $\lambda_s$  and $c_{\mr{RE}}>1$, we have
					\begin{align}
						&\frac{2}{b} \left(r_{a,\Sigma}+r_{b,\Sigma} + c_{r_1}\sqrt{s}\lambda_s\right)\leq 12c_s\frac{c_{\mr{RE}}+1}{c_{\mr{RE}}-1} r_{a,\Sigma}=  12 c_{r_1}\sqrt{s}\lambda_s,
						\end{align}
						and we see the condition about $r_\Sigma$ is satisfied. From the definition, conditions about $r_1$ and $r_2$ are clearly satisfied.
			
\end{document}